\documentclass[11pt,letterpaper]{mystyle}

\newif\ifshowchanges
\showchangesfalse
\newcommand{\change}[1]{\ifshowchanges\textcolor{blue}{#1}\else#1\fi}

\usepackage{bbm}
\tcbuselibrary{skins}
\providecommand{\mathbbm}[1]{\mathbb{#1}}
\providecommand{\keywords}[1]{}
\definecolor{ForestGreen}{rgb}{0.13, 0.55, 0.13}
\definecolor{RoyalBlue}{rgb}{0.25, 0.41, 0.88}
\definecolor{cvprpurple}{rgb}{0.6, 0.4, 0.8}

\title{\textit{V2VCrafter: Consistent Street-View Image Generation Across Vehicles}}
\date{\monthyeardate\today}
\renewcommand{\preabstractfigure}{
\begin{figure}[H]
    \centering
    \includegraphics[width=1\textwidth]{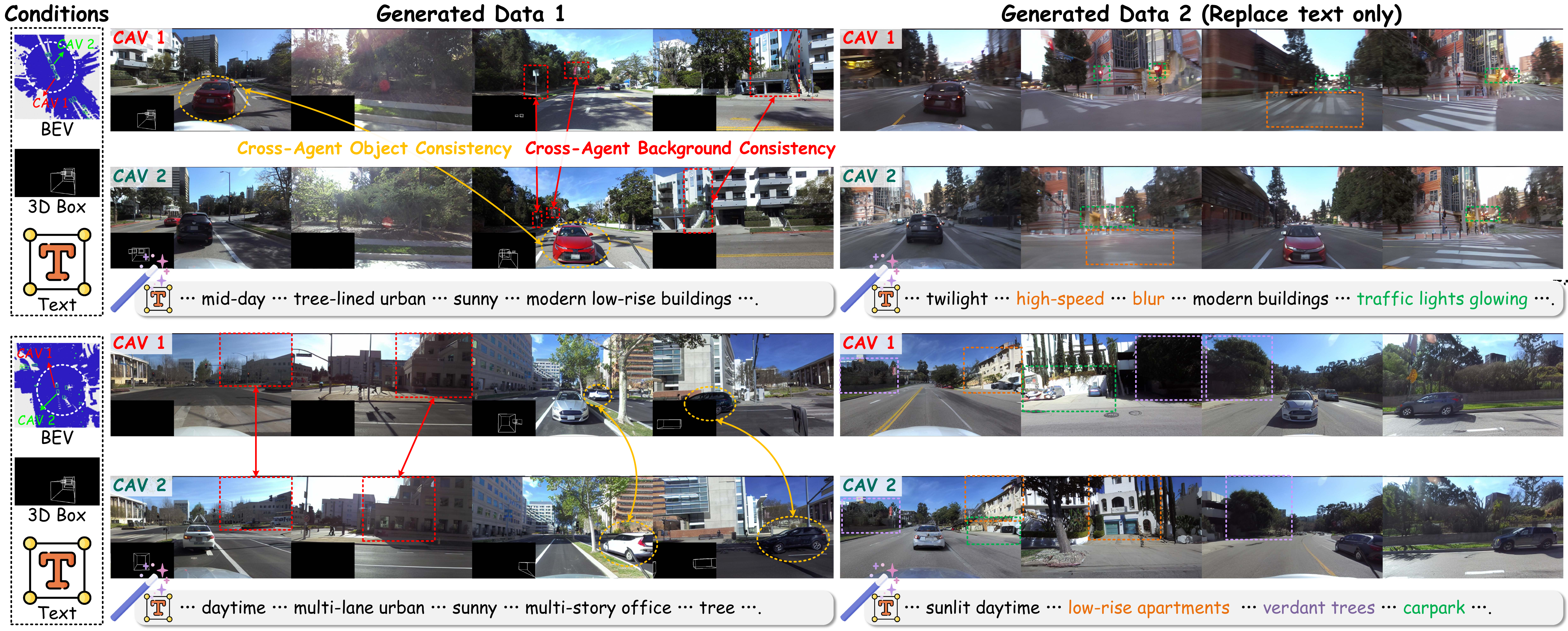}
    \caption{We present \texttt{V2VCrafter}, a street-view image generation framework across vehicles. It is capable of generating consistent content for both foreground objects ({\color{orange}orange ellipses}) and background appearance ({\color{red}red rectangles}) across multiple vehicles' perspectives (left), and can perform controllable generation via text commands (right).}
    \label{fig:comparison}
    \vspace{-5mm}
\end{figure}
}

\author{
  Yihang Tao$^{1,2,*}$,
  Yu Guo$^{1,2,*}$,
  Senkang Hu$^{1,2}$,
  Yanan Ma$^{1,2}$,
  Zihan Fang$^{1,2}$,
  Sam Tak Wu Kwong$^{3}$,
  Yuguang Fang$^{1,2}$
\\
\normalfont{
$^1$ Hong Kong JC STEM Lab of Smart City,
$^2$ City University of Hong Kong,
$^3$ Lingnan University
}}

\begin{document}

\begin{abstract}

Connected and autonomous driving (CAD) systems leverage vehicle-to-vehicle (V2V) communication for multi-agent collaborative perception, yet remain constrained by scarce annotated real-world V2V datasets and limited generalization across diverse driving conditions.
While image generation offers a feasible solution for data augmentation, existing single-vehicle multi-view generation frameworks face two key challenges in multi-agent settings: (1) the expanded learning objective degrades generation quality, and (2) dynamic inter-agent variation hinders consistency modeling for physical attributes (e.g., color, category) of jointly observed objects.
To bridge this gap, we propose \texttt{V2VCrafter}, the first framework for generating controllable and realistic multi-view driving images across vehicles.
For effective learning, we develop a progressive multi-agent diffusion model based on a single-agent backbone, using neighboring agents' latent states to progressively guide single-to-multi-agent generation.
To address cross-vehicle inconsistency, we further propose a cross-agent attention module that leverages a collaboration view graph and learnable jointly observed object representations to model dynamic cross-vehicle camera view relationships.
Experiments on real-world V2X-Real dataset show that \texttt{V2VCrafter} generates high-fidelity, controllable, and consistent street views across vehicles, thereby effectively enhancing downstream collaborative 3D object detection tasks.

\end{abstract}
\maketitle
\begingroup
\renewcommand{\thefootnote}{\fnsymbol{footnote}}
\footnotetext[1]{These authors contributed equally.}
\endgroup

\pagestyle{headstyle}
\thispagestyle{empty}

\vspace{-5mm}
\section{Introduction}
\label{sec:intro}

Connected and autonomous driving (CAD) systems enhance driving safety by leveraging Vehicle-to-Vehicle (V2V) communication for Collaborative Perception (CP) among multiple vehicles ~\citep{10979246, fangPACPPriorityAwareCollaborative2024, wei2024editable, tao2025gcpguardedcollaborativeperception, hu2024cpguardmaliciousagentdetection, ma2025sense4flvehicularcrowdsensingenhanced, tao2026learningmutualviewinformation, tao2024directcpdirectedcollaborativeperception}.
While deep learning-based CP models have advanced rapidly, their generalization ability is still heavily constrained by the scale and diversity of available annotated training data. 
Specifically, the acquisition and annotation process is costly and labor-intensive. Different from single-vehicle scenarios~\citep{Sun_2020_CVPR}, data collection from multiple connected and autonomous vehicles (CAVs) requires more significant operational resources for sensor calibration, data synchronization, and precise spatiotemporal alignment, resulting in datasets that are orders of magnitude scarcer~\citep{v2x-real}. 
This inherent scarcity is further exacerbated by the substantial labor required for annotating 3D scenes from multiple viewpoints.
Furthermore, existing datasets exhibit large disparities in training difficulty across multiple attributes, including different traffic scenarios, weather conditions, and object distributions.
Stemming from skewed sample distributions and systematic data collection biases, these imbalances compromise the model's ability to generalize to rare and complex CAD conditions.

To address these challenges, current efforts in CAD domain have primarily relied on data augmentation with simulation platforms~\citep{10161384, 10979246, Pan_2025_CVPR, fang2026agent}. These methods programmatically generate new training scenarios by altering parameters such as agent poses and communication conditions within the simulator like CARLA \citep{Dosovitskiy17}. 
However, these approaches are constrained to virtual environments and thus suffer from a persistent sim-to-real domain gap. Building upon recent advancements in diffusion models for autonomous driving street-view generation~\citep{gao2023magicdrive, wang2024drive}, another feasible solution is image generation-driven data augmentation.
Such methods typically condition the generation process on camera poses and their adjacency relationships, as well as layout information such as bounding boxes and road maps~\citep{li2023gligen, Zheng_2023_CVPR}. By applying a joint initial noise across multiple views, they learn to denoise all perspectives simultaneously to generate a coherent, multi-view scene. 

Despite this impressive progress, the transfer of these single-agent methods to multi-agent settings poses challenges due to inherent differences in system complexity.
The primary obstacles for such an extension are twofold. \ding{182} \textit{First, the expansion of the learning objective.} The increased number of views from multiple agents makes it more difficult to converge (see Appendix~\ref{app:training_objective}), as the shared-noise, one-shot denoising process becomes difficult to optimize when conditioned on a large and variable number of distinct multi-agent viewpoints, leading to degraded generation quality.
\ding{183} \textit{Second, the highly dynamic variations across agents.} Existing models are designed for a fixed intra-agent camera view pairs and are unable to model the dynamic cross-agent camera view relationships that change as CAVs move. This inability to enforce consistency across a dynamic connection of agents makes it hard to generate a coherent global scene, which is critical for CAD.

In response to the challenges outlined, we introduce \texttt{V2VCrafter}, the first framework for realistic and controllable scene generation for CAD across multiple CAVs. Instead of learning to denoise all views from multiple CAVs at once, we propose a progressive multi-agent diffusion model built upon a shared single-agent diffusion backbone. This approach overcomes the challenges on the expansion of the learning objective by reformulating the generation as a two-stage progressive process: it first masters single-agent scene generation and then uses the latent states of neighboring CAVs as reference signals to progressively guide the diffusion from a single-agent to a multi-agent context. To further enhance cross-agent consistency, we design a novel cross-agent attention module. This module effectively models the dynamic spatial relationships between agents via a collaboration view graph and learnable jointly observed object representation, ensuring that the generated views are coherent across agents. Representative results are shown in Fig. \ref{fig:comparison}.

The main contributions of this work are threefold: 
\begin{itemize} 
\item We propose \texttt{V2VCrafter}, the first framework dedicated to realistic, controllable and consistent street view image generation across multiple CAVs.
\item We introduce a novel multi-agent driving diffusion model for effective single-to-multi-agent expansion, along with a cross-agent attention module that ensures consistency across dynamic cross-agent camera views. 
\item Experiments have shown that our method generates high-fidelity and coherent collaborative driving scenes, and effectively boosts the performance of downstream collaborative 3D object detection tasks.
\end{itemize}

\section{Related Work}

\textbf{Diffusion Models for Conditional Generation.}
Diffusion models have shown strong cross-modal generation capabilities, particularly for Text-to-Image (T2I) synthesis \citep{NEURIPS2019_1d72310e, Zhang_2023_ICCV, NEURIPS2022_ec795aea}. Pioneering works such as DALL-E 2 \citep{ramesh2022hierarchicaltextconditionalimagegeneration} and Imagen \citep{NEURIPS2022_ec795aea} leverage text encoders, while Stable Diffusion \citep{rombach2021highresolution} improves efficiency via latent compression. Yet text alone cannot specify precise geometry (e.g., object positions and orientations). Layout-to-Image (L2I) methods therefore introduce explicit structural constraints. LayoutDiff \citep{Zheng_2023_CVPR}, GLIGEN \citep{li2023gligen}, and Neptune-X \citep{guo2025neptune-x} use bounding boxes or masks for precise instance-level control, but treat observers independently and overlook complex multi-observer relationships. Beyond synthesis, diffusion-generated data from geometric annotations can also benefit downstream tasks such as 2D object detection \citep{chen2023geodiffusion, guo2025neptune-x}. Motivated by this, we study L2I diffusion for multi-agent collaborative driving scenes to improve downstream collaborative 3D object detection.

\noindent\textbf{Data Augmentation for CAD.}
To improve robustness under diverse environments, prior work has explored augmenting CAD datasets. V2XP-ASG \citep{10161384} introduces the first framework for automatically generating challenging scenarios by perturbing multiple agents' LiDAR poses in adversarial yet physically plausible ways and rendering them in CARLA. Similarly, V2XVerse \citep{10979246} supports dynamic editing of communication parameters in CARLA-based simulations for systematic robustness evaluation under practical issues such as communication latency and pose estimation errors.
\change{Most recently, TYP~\citep{Pan_2025_CVPR} explores generative models for collaborative perception. However, it targets sparse LiDAR point clouds and relies on simulation datasets (created by CARLA) to learn perspective transfer from single-vehicle data.}
Despite these advances, current techniques rely on simulation-based environments \change{or focus on sparse geometric data}, largely leaving real-world CAD augmentation \change{with dense, photorealistic multi-agent camera images} unexplored.

\noindent\textbf{Multi-Camera Street View Generation.}
For outdoor driving scene generation, BEVGen~\citep{swerdlow2024streetviewimagegenerationbirdseye} uses an auto-regressive transformer with cross-view attention, BEVControl~\citep{yang2023bevcontrolaccuratelycontrollingstreetview} adds cross-view and cross-object attention, MagicDrive~\citep{gao2023magicdrive} and MagicDrive-v2~\citep{gao2025magicdrive-v2} inject cross-view attention into UNet or DiT blocks.
However, these designs target single-agent settings, enforcing consistency across fixed overlapping camera pairs and thus remaining ill-suited to dynamic, unstructured cross-agent camera views. Moreover, jointly denoising a large and variable number of views makes multi-agent driving scene generation harder to train. 

\section{Problem Formulation}
\label{sec:problem_setup}

\textbf{CAD Scene Representation.}
We model a CAD scene with $N$ agents as $\mathcal{A} = \{A_1, \dots, A_N\}$. Each agent $A_i$ has a global position $L_i$ and is equipped with a set of cameras $\mathcal{C}_i = \{C_{i,1}, \dots, C_{i,M_i}\}$. Each camera $C_{i,j}$ has a pose $\mathbf{P}_{i,j} = [\mathbf{K}_{i,j}, \mathbf{R}_{i,j}, \mathbf{t}_{i,j}]$, where $\mathbf{K}_{i,j}$ is the intrinsic matrix, and $(\mathbf{R}_{i,j}, \mathbf{t}_{i,j})$ are the extrinsic parameters (rotation and translation). The shared environment is described by $\mathcal{S} = \{\mathcal{O}, \mathcal{L}\}$, where $\mathcal{O}$ represents the set of 3D objects (each with bounding box and semantic class) visible in the scene, and $\mathcal{L}$ contains global attributes like weather and street background.
To manage the complex spatial relationships among cameras, we define a \textit{collaboration view graph} $\mathcal{G} = (\mathcal{V}, \mathcal{E})$. The vertices $\mathcal{V} = \bigcup_{i=1}^N \mathcal{C}_i$ are the set of all cameras, and the edges $\mathcal{E}$ connect any two cameras with overlapping fields of view (FoV). We perform an edge-disjoint 2-decomposition\footnote{An edge-disjoint k-decomposition of a graph is the partitioning of its edges into k edge-disjoint subgraphs. Here, k=2.} of the graph $\mathcal{G}$ into an intra-agent view graph $\mathcal{G}_{\text{in}}$ and a cross-agent view graph $\mathcal{G}_{\text{cr}}$.
The overall goal is to train a generative model $G$ that synthesizes a set of realistic and consistent images $\{\mathbf{I}_{i,j}\}$ for all cameras, conditioned on this comprehensive scene representation.

\noindent \textbf{Progressive Single-to-multi-agent Training.}
Instead of directly learning multi-agent multi-camera image denoising, we decompose the problem into a two-stage progressive process. This involves first learning single-agent multi-camera denoising, and then learning to guide this process with cross-agent context. This decomposition is realized through a two-stage progressive training with a unified objective:
\begin{equation}
    \mathcal{L} = \mathbb{E}_{z_{i,0}, \epsilon, t} \left[ ||\epsilon - \epsilon_\theta(z_{i,t}, t, \underbrace{\{\mathcal{S}, \mathbf{P}_i, \mathcal{G}_{\text{in}}(i)\}}_{\text{Intra-Agent Info.}}, \underbrace{\kappa_\text{ref}}_{\text{C.A. Ref.}})||^2 \right],
\end{equation}
where $z_{i,t}$ is the noisy latent for agent $A_i$'s views. The training has two stages:

\noindent\textit{Stage I: Intra-agent Multi-conditional Diffusion.} The model is first trained using a simplified version of the objective where the cross-agent reference (C.A. Ref.) term $\kappa_\text{ref}$ is set to $\varnothing$. In this stage, the network learns single-agent multi-camera view generation based solely on its own information: camera poses ($\mathbf{P}_i$), geometric conditions (including text descriptions, 3D bounding boxes, road map, etc.), and the intra-agent view graph ($\mathcal{G}_{\text{in}}(i)$). This ensures generating coherent multi-view images based on basic conditions for an individual agent.

\noindent\textit{Stage II: Cross-agent Context-guided Diffusion.} Building upon the pre-trained model, this stage uses the full objective shown above, introducing the cross-agent reference $\kappa_{\text{ref}} = \{\mathcal{H}_{\text{ref}}(i), \mathcal{G}_{\text{cr}}(i)\}$. This term consists of the latent states of collaborative agents ($\mathcal{H}_{\text{ref}}(i)$) and the cross-agent view graph ($\mathcal{G}_{\text{cr}}(i)$). Crucially, the latent states are obtained by encoding the neighbors' data using the model pre-trained in Stage I. This stage adapts the model to generate consistent views across multiple agents by conditioning on both the neighbors' latent states and their spatial connections. A detailed analysis of why this progressive decomposition improves trainability is provided in Appendix~\ref{app:training_objective}.

\begin{figure*}[t]
    \centering
    \includegraphics[width=1\textwidth]{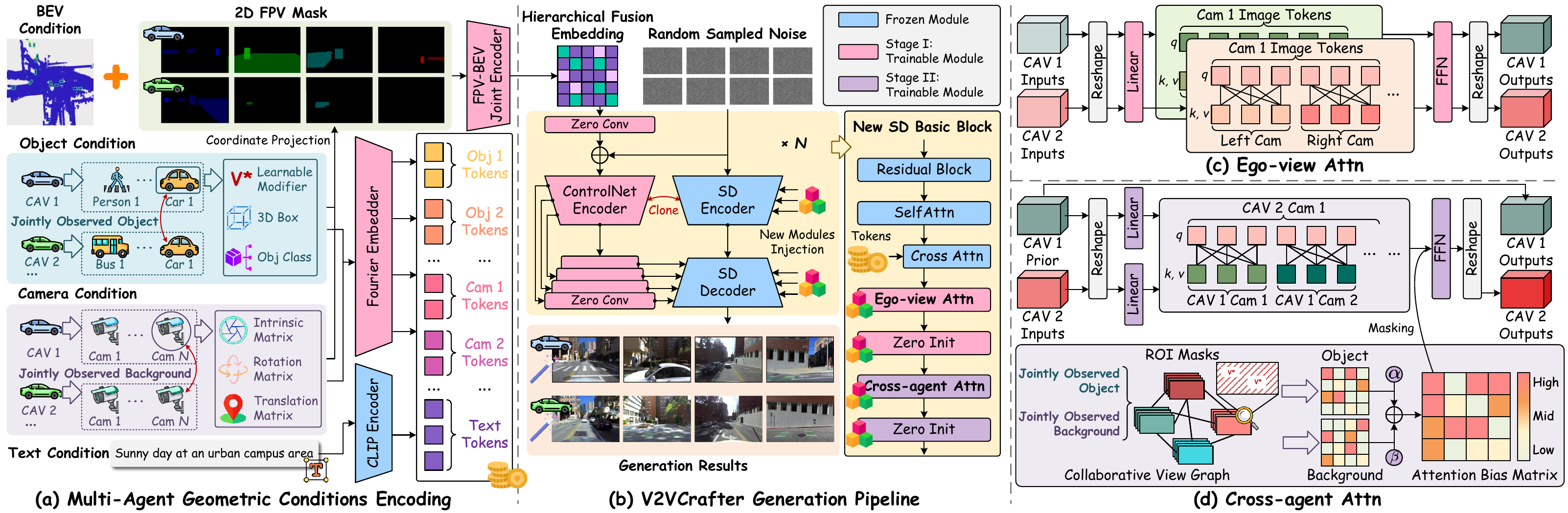}
    \caption{\textbf{Architecture of our \texttt{V2VCrafter}}. (a) Multi-agent geometric conditions encoding with hierarchical FPV-BEV fusion and learnable modifier $\mathbf{V}^*$ for jointly observed objects. (b) Two-stage progressive training pipeline. (c) Ego-view attention for intra-agent consistency. (d) Cross-agent attention with collaboration view graph and ROI masking for cross-agent consistency.}
    \label{fig:method}
    \vspace{-5mm}
\end{figure*}

\section{V2VCrafter}
\label{sec:method}

The architecture of our proposed \texttt{V2VCrafter} is illustrated in Fig.~\ref{fig:method}. It features a two-stage progressive learning framework and a cross-agent attention module for realistic and controllable multi-agent driving scene generation. This section details the key components of our approach.

\subsection{Multi-agent Geometric Conditions Encoding}
\label{sec:conditioning}
As shown in Fig.~\ref{fig:method}a, the multi-agent geometric condition encoding strategy is used to help the model better understand the complexities of CAD. For scene-level encoding, we retain strategies similar to those in single-agent autonomous driving. Specifically, we use the CLIP encoder for text descriptions and represent camera poses ($\mathbf{P}$) using Fourier features to provide global context. However, to respond to the unique challenges of the multi-agent setting, we introduce two additional designs.

\noindent \textbf{Shared Object Encoding with Learnable Modifier.} For each agent $i$, we represent the object condition as a sequence $\mathcal{O}_i = \{(\mathbf{b}_k, l_k)\}_{k=1}^{K_i}$, where $\mathbf{b}_k \in \mathbb{R}^{8 \times 3}$ encodes the 3D bounding box corners and $l_k$ is the semantic class. To explicitly distinguish jointly observed objects (the same 3D object observed by multiple agents at the same time), we introduce a learnable modifier embedding $\mathbf{V}^* \in \mathbb{R}^{d}$ that is additively combined with class embeddings of jointly observed objects. Let $\mathcal{A}_k$ denote the set of agents observing object $k$, the enriched class embedding becomes:
\begin{equation}
    \tilde{\mathbf{l}}_k = \mathbf{l}_k + \mathbbm{1}_{|\mathcal{A}_k|>1} \cdot \mathbf{V}^*,
\end{equation}
where $\mathbf{l}_k$ is the CLIP-initialized class embedding and $\mathbbm{1}_{|\mathcal{A}_k|>1}$ is an indicator function for jointly observed objects. The final object token is computed as $\mathbf{e}_k = \text{MLP}([\mathcal{F}(\mathbf{b}_k); \tilde{\mathbf{l}}_k])$, where $\mathcal{F}$ denotes Fourier positional encoding \citep{hua2025fourier}. We embed $\mathbf{V}^*$ as a learnable parameter within the trainable bounding box encoder of the ControlNet module \citep{Zhang_2023_ICCV}. These object tokens are then concatenated with text and camera embeddings to form the encoder hidden states, which are injected into the frozen UNet via its cross-attention layers, as shown in Fig.~\ref{fig:method}b. Unlike textual inversion~\citep{kumari2022customdiffusion} requiring trainable cross-attention parameters, our approach enables $\mathbf{V}^*$ to be optimized through backpropagation as gradients flow from the generation loss through the frozen cross-attention layers back to the trainable bounding box encoder. This design allows $\mathbf{V}^*$ to modulate jointly observed object representations while maintaining full compatibility with the frozen Stable Diffusion backbone, providing semantic-level control that complements the spatial-level guidance introduced below.

\noindent \textbf{Hierarchical FPV-BEV Joint Encoder.} Implicit 3D box encoding with Bird's-Eye-View (BEV) maps \cite{gao2023magicdrive} suffers from limited precision in object localization and poor convergence under data-constrained settings, making it less suitable for our multi-agent driving scenarios. To address this, we introduce explicit per-camera First-Person-View (FPV) semantic masks that provide direct spatial supervision and work jointly with global BEV maps in a hierarchical encoding scheme.
For each camera $C_{i,j}$, we generate FPV masks by projecting object bounding boxes onto the image plane. Given a 3D box $\mathbf{b} = (\mathbf{p}, \mathbf{s}, \theta)$ with center $\mathbf{p} \in \mathbb{R}^3$, size $\mathbf{s} \in \mathbb{R}^3$, and yaw $\theta$, its corner points $\{\mathbf{v}_k\}_{k=1}^8$ are projected to 2D coordinates $\mathbf{u}_k = [x_1/x_3, x_2/x_3]^\top$ using camera intrinsics $\mathbf{K}_{i,j}$ and extrinsics $[\mathbf{R}_{i,j} | \mathbf{t}_{i,j}]$, where the homogeneous coordinates are calculated as:
\begin{equation}
    [x_1, x_2, x_3]^\top = \mathbf{K}_{i,j} (\mathbf{R}_{i,j} \mathbf{v}_k + \mathbf{t}_{i,j}).
\end{equation}
The resulting vector $[x_1, x_2, x_3]^\top$ represents the point in homogeneous image coordinates, where $x_1, x_2$ are the pixel coordinates and $x_3$ is the depth.
The projected polygon is rasterized using convex hull operations to form a per-camera mask $\mathbf{M}_{i,j} \in \mathbb{R}^{H \times W}$, where each pixel encodes the object semantic class.

To integrate complementary information from global road structure and local object placement, we employ dual encoders $\mathcal{E}_{\text{BEV}}$ and $\mathcal{E}_{\text{FPV}}$ that extract features at different semantic levels. Let $\mathbf{B} \in \mathbb{R}^{C_b \times H_b \times W_b}$ denote the BEV map and $\mathbf{M} = \{\mathbf{M}_{i,j}\}$ denote the set of all per-camera masks:
\begin{equation}
    \mathbf{F}_{\text{BEV}} = \mathcal{E}_{\text{BEV}}(\mathbf{B}), \quad \mathbf{F}_{\text{FPV}} = \mathcal{E}_{\text{FPV}}(\mathbf{M}).
\end{equation} The dual latent are fused via a weighted residual connection:
\begin{equation}
    \mathbf{F}_{\text{cond}} = \mathbf{F}_{\text{FPV}} + \eta \cdot \mathbf{F}_{\text{BEV}},
\end{equation}
where $\eta$ controls the contribution of BEV features. Although FPV masks and BEV layouts are defined in different coordinate systems, the two encoders first project them into a shared latent space before fusion. FPV mask provides precise object placement, while BEV map supplements global road context. As shown in Appendix~\ref{app:ablation}, a small residual weight works better. The fused features are then injected into the UNet decoder blocks.

\subsection{Ego-view Attention}
\label{sec:pretraining}
In the Stage I training, we establish a robust single-agent multi-camera generative model, as shown in Fig.~\ref{fig:method}c. This model leverages the ego-agent's camera arrangement, which defines a static intra-agent view graph $\mathcal{G}_{\text{in}}(i) = (\mathcal{V}_i, \mathcal{E}_{\text{in}})$. Here, $\mathcal{V}_i = \mathcal{C}_i$ are the camera vertices, and the edges $\mathcal{E}_{\text{in}} \subseteq \mathcal{V}_i \times \mathcal{V}_i$ encode the fixed spatial adjacency based on FoV overlap. 
For each target camera $c \in \mathcal{V}_i$, its neighborhood is defined by the graph structure: $\mathcal{N}_{\text{in}}(c) = \{c' \in \mathcal{V}_i \mid (c, c') \in \mathcal{E}_{\text{in}}\}$. Given the latent feature representation $\mathbf{f}_c \in \mathbb{R}^{L \times D}$ for camera $c$, we perform a \textit{static message passing} process via a neighbor attention mechanism to facilitate information flow. The feature aggregation for camera $c$ is formulated as:
\begin{equation}
    \mathbf{f}'_c = \mathbf{f}_c + \sum_{j \in \mathcal{N}_{\text{in}}(c)} \mathbf{A}(c, j) \odot \phi_V(\mathbf{f}_j),
\end{equation}
where $\phi_V: \mathbb{R}^{L \times D} \to \mathbb{R}^{L \times D}$ is a learnable value projection, and the attention coefficient $\mathbf{A}(c, j) \in \mathbb{R}^{L \times L}$ is computed:
\begin{equation}
    \mathbf{A}(c, j) = \text{Softmax}\left(\frac{\phi_Q(\mathbf{f}_c) \cdot \phi_K(\mathbf{f}_j)^\top}{\sqrt{d_k}}\right),
\end{equation}
with $\phi_Q, \phi_K: \mathbb{R}^{L \times D} \to \mathbb{R}^{L \times d_k}$ being query and key projections, respectively. 
This design enables a static information flow along the fixed edges of $\mathcal{G}_{\text{in}}$, enforcing spatial coherence across neighboring views. The pre-defined sparse connectivity ensures computation efficiency while capturing the essential geometric constraints of the single-agent setup. The updated features $\{\mathbf{f}'_c\}_{c \in \mathcal{V}_i}$ are then processed by the UNet decoder to generate multi-view images for agent $i$.

\subsection{Cross-agent Attention}
\noindent\textbf{Dynamic Collaboration View Graph Construction.}
The complete collaboration view graph $\mathcal{G} = (\mathcal{V}, \mathcal{E})$ comprises two disjoint subgraphs: the static intra-agent view graph $\mathcal{G}_{\text{in}}$ (used in the ego-view attention) and the dynamic cross-agent view graph $\mathcal{G}_{\text{cr}}$, where $\mathcal{V} = \bigcup_{i=1}^N \mathcal{C}_i$ and $\mathcal{E} = \mathcal{E}_{\text{in}} \cup \mathcal{E}_{\text{cr}}$. 
To enable our cross-agent attention, we first construct $\mathcal{G}_{\text{cr}} = (\mathcal{V}, \mathcal{E}_{\text{cr}})$ where $\mathcal{E}_{\text{cr}} \subseteq \mathcal{V} \times \mathcal{V}$ captures inter-agent spatial relationships. Let $\mathcal{A}(c)$ denote the agent owning camera $c$, and $\mathcal{O}_{c}$ the set of visible object IDs in $c$. We define auxiliary predicates: $\Delta(c_i, c_j) := \mathbbm{1}_{\mathcal{A}(c_i) \neq \mathcal{A}(c_j)}$ for cross-agent constraint, and $\Omega(c_i, c_j) := \mathbbm{1}_{|\mathcal{F}_{c_i} \cap \mathcal{F}_{c_j}| > 0}$ for geometric overlap where $\mathcal{F}_{c}$ denotes camera $c$'s frustum footprint on the BEV plane. We construct $\mathcal{E}_{\text{cr}}$ through two steps:

\noindent \textit{Step I: Object-based Connectivity Analysis.} We first establish connections between cameras that share at least one jointly observed object, as these views explicitly exhibit strong semantic correlation and spatial overlap. This stage leverages object ID annotations commonly available in CAD datasets to identify camera pairs that observe the same entities:
\begin{equation}
    E_{\text{obj}} = \{ (c_i, c_j) \mid \Delta(c_i, c_j) \cdot \mathbbm{1}_{\mathcal{O}_{c_i} \cap \mathcal{O}_{c_j} \neq \emptyset} = 1 \}.
\end{equation}

\noindent \textit{Step II: Geometric Overlap Reasoning.} For camera pairs not connected in Step I, we perform geometric reasoning by projecting their viewing frustums onto the BEV plane. This stage captures spatial relationships between cameras that may not observe common objects but still have overlapping FoV, which is crucial for maintaining global spatial coherence:
\begin{equation}
    E_{\text{geo}} = \{ (c_i, c_j) \notin E_{\text{obj}} \mid \Delta(c_i, c_j) \cdot \Omega(c_i, c_j) = 1 \}.
\end{equation}
The final edge set is $\mathcal{E}_{\text{cr}} = E_{\text{obj}} \cup E_{\text{geo}}$, yielding adjacency matrix $\mathbf{W} \in \{0,1\}^{|\mathcal{V}| \times |\mathcal{V}|}$ with $W_{ij} = \mathbbm{1}_{(c_i, c_j) \in \mathcal{E}_{\text{cr}}}$.

\noindent\textbf{Jointly Observed ROI Masking.}
To enforce object-level consistency for jointly observed entities, we introduce a structured attention bias that guides cross-agent feature fusion. For each connected camera pair $(c_i, c_j) \in \mathcal{E}_{\text{cr}}$, we partition the token space into three categories based on their spatial correspondence to 3D bounding boxes and object visibility: (1) $\mathbf{V}^*$ object tokens corresponding to the shared objects marked by the learnable modifier $\mathbf{V}^*$, which require strict cross-agent appearance consistency; (2) non-$\mathbf{V}^*$ object tokens representing vehicle-specific objects visible only to one agent; and (3) background tokens covering road, buildings, and sky regions. 
We construct a dual-component attention bias matrix $\mathbf{M}_{ji} \in \mathbb{R}^{L \times L}$ that differentially regulates attention flow across token types:
\begin{equation}
    \mathbf{M}_{ji} = \alpha \cdot \mathbf{S}_{ji} + \beta \cdot \mathbf{B}_{ji},
\end{equation}
where $\mathbf{S}_{ji}$ and $\mathbf{B}_{ji}$ serve different roles, respectively: $\mathbf{S}_{ji}$ uses strong suppression ($-\tau_o$) to isolate $\mathbf{V}^*$ objects, ensuring each shared object only attends to its corresponding instance in other vehicles while blocking attention to unrelated objects; $\mathbf{B}_{ji}$ applies weaker suppression ($-\tau_b$) to regularize background and non-$\mathbf{V}^*$ tokens. 

We define the following indicator functions: $\mathbbm{1}_s(k,l)$ indicates whether tokens $k$ and $l$ represent the same shared object instance; $\mathbbm{1}_v(k)$ indicates $\mathbf{V}^*$ membership; $\mathbbm{1}_d(k,l)$ indicates whether tokens $k$ and $l$ correspond to different objects or spatial regions (i.e., they should not strongly attend to each other); and $\mathbbm{1}_{v\oplus}(k,l) = \mathbbm{1}_v(k) \oplus \mathbbm{1}_v(l)$ (XOR operation) captures the mixed $\mathbf{V}^*$/non-$\mathbf{V}^*$ pairs. The attention bias components are denoted as:
\begin{equation}
\begin{aligned}
    S_{ji}^{(k,l)} &= -\tau_o \cdot [\mathbbm{1}_{v\oplus}(k,l) + \mathbbm{1}_v(k)\mathbbm{1}_v(l)(1 - \mathbbm{1}_s(k,l))], \\
    B_{ji}^{(k,l)} &= -\tau_b \cdot (1-\mathbbm{1}_v(k))(1-\mathbbm{1}_v(l)) \mathbbm{1}_d(k,l),
\end{aligned}
\end{equation}
where $\tau_o \gg \tau_b > 0$ ensures $\mathbf{V}^*$ objects receive stronger attention control than background regions. The term $\mathbbm{1}_{v\oplus}(k,l)$ creates a bidirectional barrier preventing information leakage between $\mathbf{V}^*$ and non-$\mathbf{V}^*$ tokens, while the term $\mathbbm{1}_v(k)\mathbbm{1}_v(l)(1-\mathbbm{1}_s(k,l))$ blocks attention between different $\mathbf{V}^*$ objects. The learnable scalars $\alpha$, and $\beta$ balance the object-level consistency and the scene-level coherence.

\noindent\textbf{Cross-agent Masked Attention.}
To account for the evolving cross-agent camera view relationships, we perform \textit{dynamic message passing} on the graph $\mathcal{G}_{\text{cr}}$. For an ego agent's camera $c_j$ with features $\mathbf{f}_j \in \mathbb{R}^{L \times D}$, we aggregate information from its dynamically-determined cross-agent neighbors $\mathcal{N}_{\text{cr}}(c_j) = \{c_i \mid (c_j, c_i) \in \mathcal{E}_{\text{cr}}\}$ using a masked attention mechanism:
\begin{equation}
    \mathbf{f}''_j = \sum_{i \in \mathcal{N}_{\text{cr}}(c_j)} \text{Softmax}\left(\frac{\phi_Q(\mathbf{f}_j) \cdot \phi_K(\mathbf{f}_i)^\top}{\sqrt{d_k}} + \mathbf{M}_{ji}\right) \phi_V(\mathbf{f}_i).
\end{equation}
The attention bias $\mathbf{M}_{ji}$ modulates the attention distribution to prioritize tokens corresponding to shared objects (associated with $\mathbf{V}^*$ tokens at the semantic level), thereby establishing a multi-level consistency enforcement mechanism: the $\mathbf{V}^*$ token modulates object semantics through cross-attention in the UNet, while the ROI mask guides spatial feature alignment in the graph attention layer. The enhanced features $\mathbf{f}''_j$ are then refined through a feed-forward network:
\begin{equation}
    \tilde{\mathbf{f}}_j = \text{FFN}(\text{LayerNorm}(\mathbf{f}_j + \mathbf{f}''_j)),
\end{equation}
enabling the model to synthesize views that are geometrically and semantically consistent across agents. The overall process of this attention is shown in Fig.~\ref{fig:method}d.

\subsection{Model Training and Data Selection}

Our training strategy progressively operationalizes the unified objective function defined in Section~\ref{sec:problem_setup}.
We first train the denoising model $\epsilon_\theta$ with only the Ego-view Attention module activated, then freeze the ego-centric layers and finetune the model with the Cross-agent Attention, guided by the added cross-agent reference condition $\kappa_{\text{ref}}$. We further use high-quality data selection to prioritize challenging samples and selective classifier-free guidance to balance appearance flexibility with layout fidelity. See Appendix~\ref{app:impl_details} for details.

\begin{figure*}[t]
    \centering
    \includegraphics[width=0.96\textwidth]{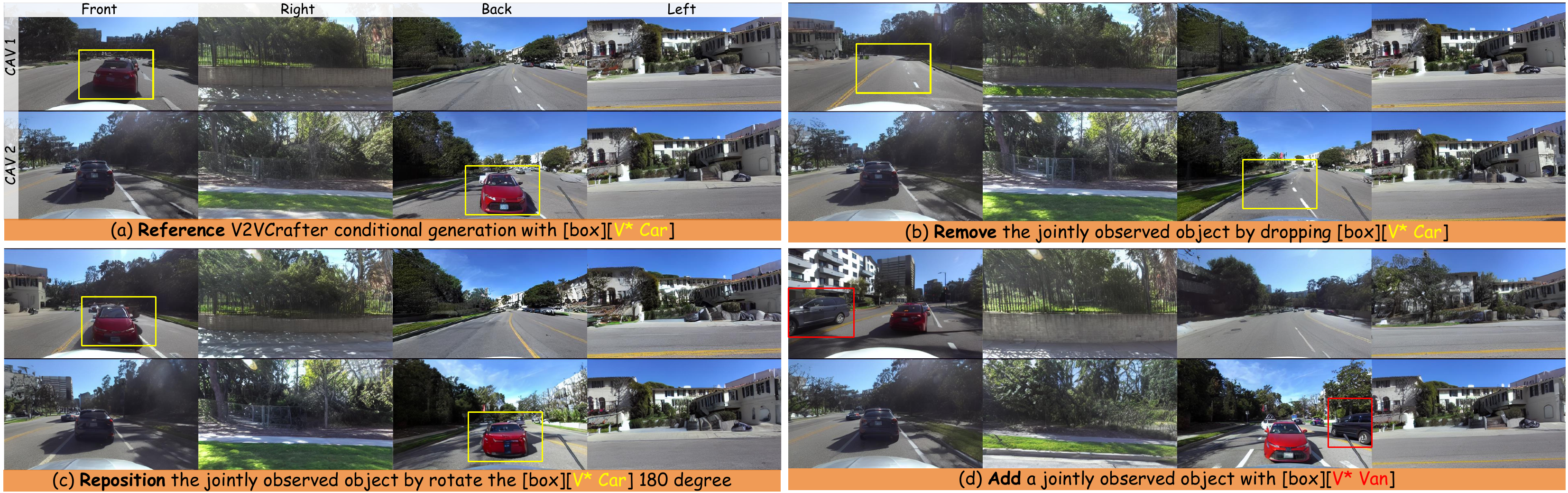} 
    \caption{\textbf{Cross-agent controllability of \texttt{V2VCrafter}}. We show generation results under diverse joint observation conditions. Rectangles with the same color denote the same observed area.}
    \label{fig:shared_object_control}
    \vspace{-4mm}
\end{figure*}

\begin{figure*}[t]
    \centering
    \includegraphics[width=0.96\textwidth]{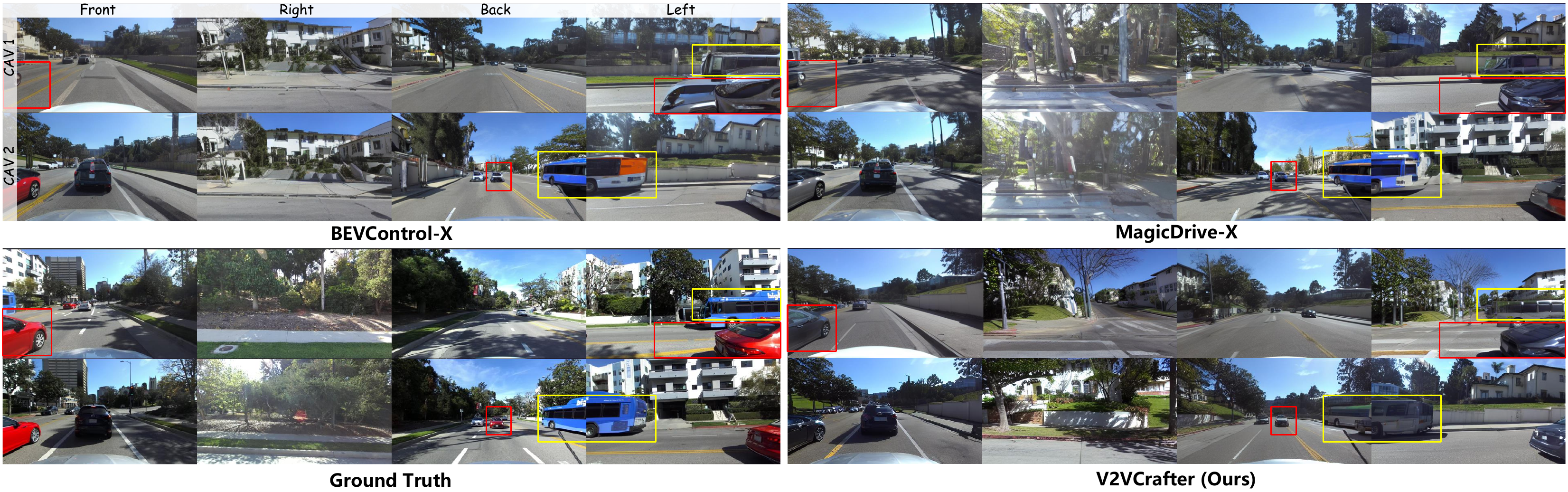}
    \vspace{-3mm}
    \caption{\textbf{Qualitative comparison on V2X-Real.} Rectangles with the same color denote the same observed area. Baselines fail to keep consistent appearance for jointly observed objects, while \texttt{V2VCrafter} maintains coherent appearance and geometry.} 
    \label{fig:gen_baselines}
    \vspace{-2mm}
\end{figure*}

\section{Experiments}

\subsection{Experimental Setup}

\textbf{Dataset.} We train and evaluate on V2X-Real dataset~\citep{v2x-real}, filtering scenes with at least two CAVs, which yields 11,808 training and 2,624 validation samples. We also generate BEV maps with OpenCOOD~\citep{xuOPV2VOpenBenchmark2022} and text descriptions with Qwen2.5-VL-7B-Instruct~\citep{bai2025qwen25vltechnicalreport}. The filtered dataset is further divided into Basic, Moderate, and Hard subsets by object distance, occlusion, and road slope, with ratios of 28.9\%, 24.4\%, and 46.7\%, respectively. The Basic split is used for image-generation evaluation, while all three splits are separately used for data-augmentation experiments. See Appendix~\ref{app:data_preprocessing} for dataset preprocessing details.

\begin{table}[t]
    \centering
    \caption{\textbf{Synthesized image evaluation.} \colorbox{red!20}{\rule{0pt}{0.7ex}\hspace{0.8em}} and \colorbox{yellow!20}{\rule{0pt}{0.7ex}\hspace{0.8em}} mark the best and second-best results.}
    \vspace{-1mm}
    \label{tab:main_results}
    \footnotesize
    \setlength{\tabcolsep}{3.5pt}
    \renewcommand{\arraystretch}{0.80}
    \resizebox{\linewidth}{!}{
    \begin{tabular}{lcccc}
        \toprule
        \multirow{2}{*}{\textbf{Method}} & \multicolumn{1}{c}{\textbf{Fidelity}} & \multicolumn{1}{c}{\textbf{Cross-agent Consistency}} & \multicolumn{2}{c}{\textbf{CoDetection Score (mAP@30/50)}} \\
        \cmidrule(lr){2-2} \cmidrule(lr){3-3} \cmidrule(lr){4-5}
         & \textbf{FID}$\downarrow$ & \textbf{Sim. / MRR / Top-1}$\uparrow$ & \textbf{Overall}$\uparrow$ & \textbf{Jointly Observed}$\uparrow$ \\
        \midrule
        \multicolumn{5}{l}{\textbf{Reference}} \\
        \textcolor{gray}{GT Ref. (Real)} & \textcolor{gray}{--} & \textcolor{gray}{0.813 / 0.656 / 0.512} & \textcolor{gray}{29.18 / 13.57} & \textcolor{gray}{33.96 / 20.15} \\
        \midrule
        \multicolumn{5}{l}{\textbf{Generative Methods}} \\
        BEVControl-X & 22.06 & \change{0.621 / 0.265 / 0.062} & 20.11 / 10.98 & 24.42 / 12.36 \\
        MagicDrive-X & \cellcolor{yellow!20}18.41 & 0.649 / 0.348 / 0.131 & \cellcolor{yellow!20}24.89 / 11.95 & \cellcolor{yellow!20}26.73 / 16.02 \\
        \change{MagicDrive-V2-X} & \change{20.91} & \cellcolor{yellow!20}\change{0.654 / 0.370 / 0.161} & 24.18 / 11.43 & 27.11 / 15.54 \\
        \textbf{V2VCrafter (Ours)} & \cellcolor{red!20}\textbf{17.46} & \cellcolor{red!20}\textbf{0.836 / 0.586 / 0.406} & \cellcolor{red!20}\textbf{27.88 / 13.21} & \cellcolor{red!20}\textbf{31.01 / 18.46} \\
        \bottomrule
    \end{tabular}}
    \vspace{-4mm}
\end{table}

\begin{table*}[t]
    \centering
    \caption{\textbf{Camera-only augmentation results.} We report Vehicle / Overall collaborative detection score on Basic / Moderate / Hard subsets. \colorbox{red!20}{\rule{0pt}{0.7ex}\hspace{0.8em}} and \colorbox{yellow!20}{\rule{0pt}{0.7ex}\hspace{0.8em}} mark the largest and second-largest gains among \textit{Real + Ours} entries in each column.}
    \label{tab:aug_cp_camera_only}
    \vspace{-3mm}
    \footnotesize
    \renewcommand{\arraystretch}{1.0}%
    \setlength{\tabcolsep}{4.5pt}
    \resizebox{\textwidth}{!}{
    \begin{tabular}{ll|ccc|ccc}
        \toprule
        \multirow{2}{*}{\textbf{CP Model}} & \multirow{2}{*}{\textbf{Training Data}} & \multicolumn{3}{c}{\textbf{Vehicle AP@30/50}} & \multicolumn{3}{c}{\textbf{Overall mAP@30/50}} \\
        \cmidrule(lr){3-5} \cmidrule(l){6-8}
        & & \textbf{Basic} & \textbf{Moderate} & \textbf{Hard} & \textbf{Basic} & \textbf{Moderate} & \textbf{Hard} \\
        \midrule
        \multirow{2}{*}{AttnFuse} & Real Only & 34.21/16.48 & 16.89/14.33 & 14.67/12.44 & 29.18/13.57 & 12.54/10.04 & 9.58/7.21 \\
         & Real + Ours & 38.68/19.84 {\scriptsize \textcolor{green!60!black}{(+4.47/+3.36)}} & \cellcolor{yellow!20}19.83/17.59 {\scriptsize \textcolor{green!60!black}{(+2.94/+3.26)}} & \cellcolor{yellow!20}16.21/11.66 {\scriptsize \textcolor{green!60!black}{(+1.54)}/\textcolor{BrickRed}{(-0.78)}} & 32.74/16.12 {\scriptsize \textcolor{green!60!black}{(+3.56/+2.55)}} & \cellcolor{yellow!20}14.97/12.18 {\scriptsize \textcolor{green!60!black}{(+2.43/+2.14)}} & \cellcolor{red!20}10.11/8.93 {\scriptsize \textcolor{green!60!black}{(+0.53/+1.72)}} \\
        \midrule
        \multirow{2}{*}{V2VNet} & Real Only & 37.86/18.23 & 18.21/15.16 & 16.43/13.88 & 31.64/15.22 & 13.42/10.61 & 10.71/8.04 \\
         & Real + Ours & \cellcolor{yellow!20}42.17/22.49 {\scriptsize \textcolor{green!60!black}{(+4.31/+4.26)}} & \cellcolor{red!20}21.82/18.67 {\scriptsize \textcolor{green!60!black}{(+3.61/+3.51)}} & \cellcolor{red!20}18.39/14.51 {\scriptsize \textcolor{green!60!black}{(+1.96/+0.63)}} & \cellcolor{yellow!20}35.27/17.94 {\scriptsize \textcolor{green!60!black}{(+3.63/+2.72)}} & \cellcolor{red!20}16.98/13.74 {\scriptsize \textcolor{green!60!black}{(+3.56/+3.13)}} & 9.98/8.71 {\scriptsize \textcolor{BrickRed}{(-0.73)}/\textcolor{green!60!black}{(+0.67)}} \\
        \midrule
        \multirow{2}{*}{Late Fusion} & Real Only & 25.19/14.56 & 13.53/11.89 & 11.27/10.13 & 20.63/10.47 & 9.88/8.31 & 7.61/6.28 \\
         & Real + Ours & \cellcolor{red!20}29.81/18.63 {\scriptsize \textcolor{green!60!black}{(+4.62/+4.07)}} & 15.86/13.61 {\scriptsize \textcolor{green!60!black}{(+2.33/+1.72)}} & 10.86/10.69 {\scriptsize \textcolor{BrickRed}{(-0.41)}/\textcolor{green!60!black}{(+0.56)}} & \cellcolor{red!20}24.36/13.58 {\scriptsize \textcolor{green!60!black}{(+3.73/+3.11)}} & 12.11/10.05 {\scriptsize \textcolor{green!60!black}{(+2.23/+1.74)}} & \cellcolor{yellow!20}8.19/6.11 {\scriptsize \textcolor{green!60!black}{(+0.58)}/\textcolor{BrickRed}{(-0.17)}} \\
        \bottomrule
    \end{tabular}}
    \vspace{-2mm}
\end{table*}

\begin{figure*}[htbp]
    \centering
    \includegraphics[width=1.0\textwidth]{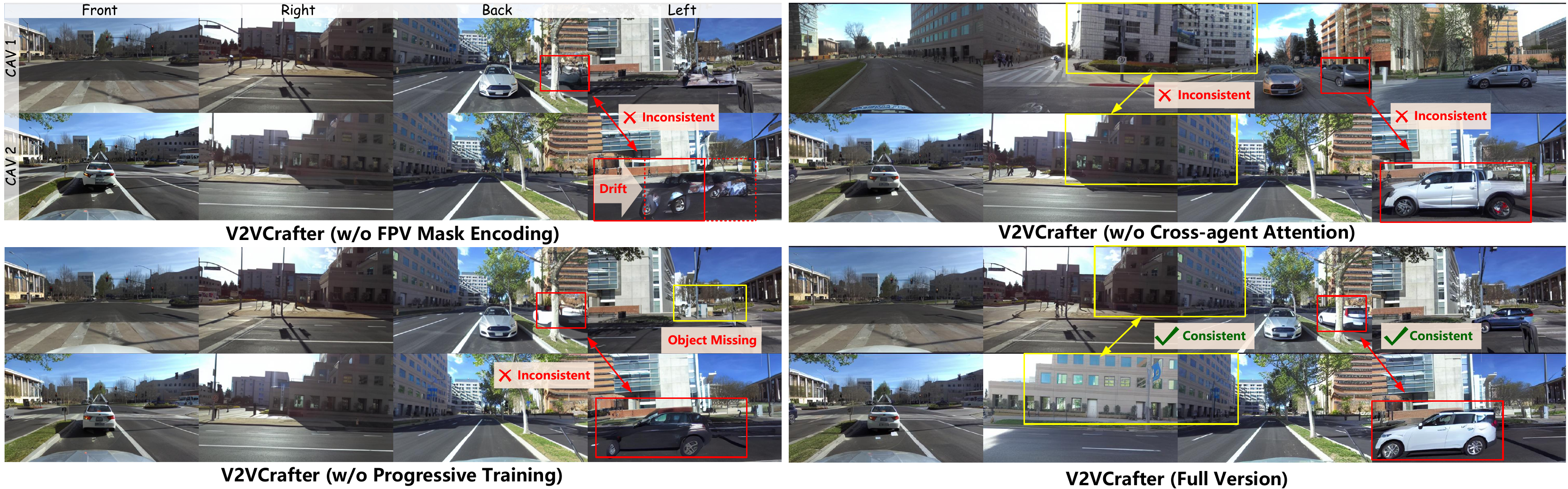}
    \vspace{-7mm}
    \caption{\textbf{Qualitative ablation study on key components}. Rectangles with the same color denote the same observed area.}
    \label{fig:ablation_qualitative}
    \vspace{-5mm}
\end{figure*}

\begin{table}[t]
    \centering
    \caption{\textbf{Ablation on image generation quality.} \colorbox{red!20}{\rule{0pt}{0.7ex}\hspace{0.8em}} and \colorbox{yellow!20}{\rule{0pt}{0.7ex}\hspace{0.8em}} mark the best and second-best results.}
    \vspace{-2mm}
    \label{tab:ablation_main}
    \footnotesize
    \renewcommand{\arraystretch}{0.80}
    \setlength{\tabcolsep}{3.4pt}
    \resizebox{\linewidth}{!}{
    \begin{tabular}{lcccc}
        \toprule
        \multirow{2}{*}{\textbf{Variant}} & \multicolumn{1}{c}{\textbf{Fidelity}} & \multicolumn{1}{c}{\textbf{Cross-agent Consistency}} & \multicolumn{2}{c}{\textbf{CoDetection Score (mAP@30/50)}} \\
        \cmidrule(lr){2-2} \cmidrule(lr){3-3} \cmidrule(lr){4-5}
         & \textbf{FID}$\downarrow$ & \textbf{Sim. / MRR / Top-1}$\uparrow$ & \textbf{Overall}$\uparrow$ & \textbf{Jointly Observed}$\uparrow$ \\
        \midrule
        \multicolumn{5}{l}{\textbf{Reference}} \\
        \textcolor{gray}{GT Ref. (Real)} & \textcolor{gray}{--} & \textcolor{gray}{0.813 / 0.656 / 0.512} & \textcolor{gray}{29.18 / 13.57} & \textcolor{gray}{33.96 / 20.15} \\
        \midrule
        \multicolumn{5}{l}{\textbf{Method Variants}} \\
        w/o CA-Attn. & 18.01 & 0.624 / 0.301 / 0.108 & 24.38 / 11.54 & 26.67 / 15.20 \\
        w/o $\mathbf{V}^*$ Embedding& \cellcolor{yellow!20}17.88 & \cellcolor{yellow!20}0.815 / 0.484 / 0.337 & \cellcolor{yellow!20}26.01 / 12.48 & \cellcolor{yellow!20}30.63 / 17.84 \\
        w/o FPV Mask & 28.26 & 0.582 / 0.247 / 0.071 & 18.42 / 10.96 & 21.76 / 11.42 \\
        w/o Prog. Training & 20.33 & 0.659 / 0.412 / 0.185 & 24.95 / 11.82 & 30.93 / 17.02 \\
        Full Model & \cellcolor{red!20}\textbf{17.46} & \cellcolor{red!20}\textbf{0.836 / 0.586 / 0.406} & \cellcolor{red!20}\textbf{27.88 / 13.21} & \cellcolor{red!20}\textbf{31.01 / 18.46} \\
        \bottomrule
    \end{tabular}}
    \vspace{-5mm}
\end{table}

\noindent \textbf{Implementation Details and Baselines.} Our model builds upon Stable Diffusion v1.5 and is trained on 8 NVIDIA A800 GPUs for 150 epochs. We use a learning rate of 1e-4, a batch size of 4 per GPU, and a CFG scale of 1.5. 
\change{For multi-agent camera image generation, we adapt MagicDrive-V2~\citep{gao2025magicdrive-v2}, MagicDrive~\citep{gao2023magicdrive}, and BEVControl~\citep{yang2023bevcontrolaccuratelycontrollingstreetview} to the multi-agent generation setting with static intra-agent cross-view attention, denoted as MagicDrive-V2-X, MagicDrive-X, and BEVControl-X.}
For the downstream collaborative 3D object detection task, we employ BEVFusion~\citep{liu2022bevfusion} as the base detector and AttnFuse~\citep{xuOPV2VOpenBenchmark2022} for collaborative fusion. 
For data augmentation experiments, we additionally evaluate CP models including V2VNet~\citep{v2vnet} and Late Fusion~\citep{11020620}.
More details are in Appendix~\ref{app:impl_details}.

\noindent \textbf{Evaluation Metrics.} We use FID to measure generation realism, and $\text{mAP}@{30/50}$ to assess collaborative detection score on overall and jointly observed objects. \change{To evaluate cross-agent consistency of jointly observed objects, we report three CLIP-based retrieval metrics \citep{pmlr-v139-radford21a}: CLIP Similarity for pairwise appearance agreement, and MRR / Top-1 Accuracy for cross-agent retrieval correctness.} For more details of the metrics, see Appendix~\ref{app:exp_results}.

\subsection{Image Generation Results}

\noindent \textbf{Cross-agent Controllability.} Beyond the basic scene-level controls shown in Fig.~\ref{fig:comparison}, we further test whether \texttt{V2VCrafter} can generate jointly observed objects while preserving cross-agent consistency. As shown in Fig.~\ref{fig:shared_object_control}, the model supports consistent object-level control, including changing the presence or position of jointly observed objects, without breaking object and scene coherence across CAVs. This behavior indicates that \texttt{V2VCrafter} enables controllable generation of jointly observed objects under multi-agent geometric constraints. Such controllability can further be used to adjust the ratio of jointly observed objects in synthesized data, whose impact on downstream collaborative perception is analyzed in Appendix~\ref{app:exp_results}.

\noindent \textbf{Baseline Comparison.} \change{We replace real validation images with synthesized ones that follow the same annotations, and evaluate them in terms of image fidelity, cross-agent consistency, and CoDetection score from a camera-only BEV detector pretrained on real data and directly applied to the synthesized inputs. Fig.~\ref{fig:gen_baselines} shows that baselines often render jointly observed objects with inconsistent appearance. Table~\ref{tab:main_results} shows that \texttt{V2VCrafter} achieves the best overall results, indicating that stronger geometric fidelity and cross-agent consistency lead to better downstream collaborative detection score. More quantitative results and video generation demonstrations are provided in Appendix~\ref{app:exp_results} and ~\ref{app:video_generation}.}

\subsection{Data Augmentation Results}
To evaluate \texttt{V2VCrafter}'s utility for downstream tasks, we mix real and generated data to train camera-only BEV detectors, and report the CoDetection Score on the validation subsets. As shown in Table~\ref{tab:aug_cp_camera_only}, the augmented data improves most metrics across three CP models, showing that the generated images provide effective supervision for downstream collaborative detection. The gains are generally larger on the Basic and Moderate subsets than on the Hard subset, likely because severe occlusion, longer observation distance, and larger road-slope variation make the hardest cases less recoverable from image augmentation alone. Camera+LiDAR data augmentation results are reported in Appendix~\ref{app:exp_results}.

\subsection{Ablation Study}

\noindent \textbf{Effectiveness of Key Modules.}
As shown in Fig.~\ref{fig:ablation_qualitative} and Table~\ref{tab:ablation_main}, removing the FPV mask causes the largest overall drop, showing that explicit FPV supervision is critical for precise spatial control. 
Removing \textit{Cross-agent Attention (CA-Attn.)} markedly hurts consistency and jointly observed codetection score, indicating that CA-Attn. is the main component for cross-agent alignment. \change{When CA-Attn. is kept but $\mathbf{V}^*$ is removed, the performance drop is smaller but still clear, suggesting that $\mathbf{V}^*$ provides additional semantic cues for identifying shared instances beyond spatial alignment alone.}

\noindent \textbf{Effectiveness of Progressive Training.} The setting ``w/o Prog. Training'' trains the model one-shot for two-agent generation. As shown in Table~\ref{tab:ablation_main}, it performs consistently worse than the full model on fidelity, consistency, and camera-only detection, showing that progressive training makes multi-agent optimization more effective.
More ablation studies on key hyperparameters including different CFG scales and FPV-BEV fusion weights, are provided in Appendix~\ref{app:ablation}.


\section{Conclusion}
In this paper, we have introduced \texttt{V2VCrafter}, the first framework for CAD image generation \change{across vehicles}. Our proposed two-stage progressive training strategy effectively addresses the expansion challenge of learning objective. Furthermore, the novel cross-agent attention mechanism ensures perceptual consistency for joint observation across CAVs. Experiments have confirmed that \texttt{V2VCrafter} generates high-fidelity, coherent scenes that boost downstream collaborative perception tasks. A current limitation is that Vehicle-to-Infrastructure (V2I) generation is not yet supported due to the limited view diversity of fixed RoadSide Units (RSUs) in current datasets, which we leave for future work.



\newpage

\bibliographystyle{refstyle}
\bibliography{main}
\clearpage
\appendix
\setcounter{secnumdepth}{1}
\setcounter{section}{0}
\renewcommand{\thesection}{\arabic{section}}
\section*{Appendix}

\section{Theoretical Statements and Proofs}
\label{app:training_objective}

This section presents a more formal analysis of our progressive training strategy. We first provide empirical motivation, and then organize the argument through definitions, assumptions, lemmas, the main theorem, and proofs. The goal is to clarify why decomposing multi-agent generation into intra-agent pretraining and cross-agent refinement reduces optimization difficulty. Our analysis is local to the training objective and is not intended as a global convergence guarantee for non-convex diffusion optimization.

\subsection{Empirical Motivation}
\label{subsec:challenge_joint}

\noindent\textbf{Empirical Evidence.} We train MagicDrive~\citep{gao2023magicdrive} with identical hyperparameters on two V2X-Real variants: (i) single-agent data with $N=1$ (1$\times$4 views), derived by masking CAV2 data, and (ii) dual-agent data with $N=2$ (2$\times$4 views), using complete V2X-Real. As shown in Fig.~\ref{fig:scalability}, single-agent training converges by epoch 150 and produces clear street-view images with stable cross-view geometry. In contrast, direct dual-agent training remains blurry and inconsistent even after the same number of epochs. This observation motivates a formal analysis of why naive joint multi-agent denoising is harder to optimize.

\begin{figure*}[htbp]
    \centering
    \includegraphics[width=1.0\textwidth]{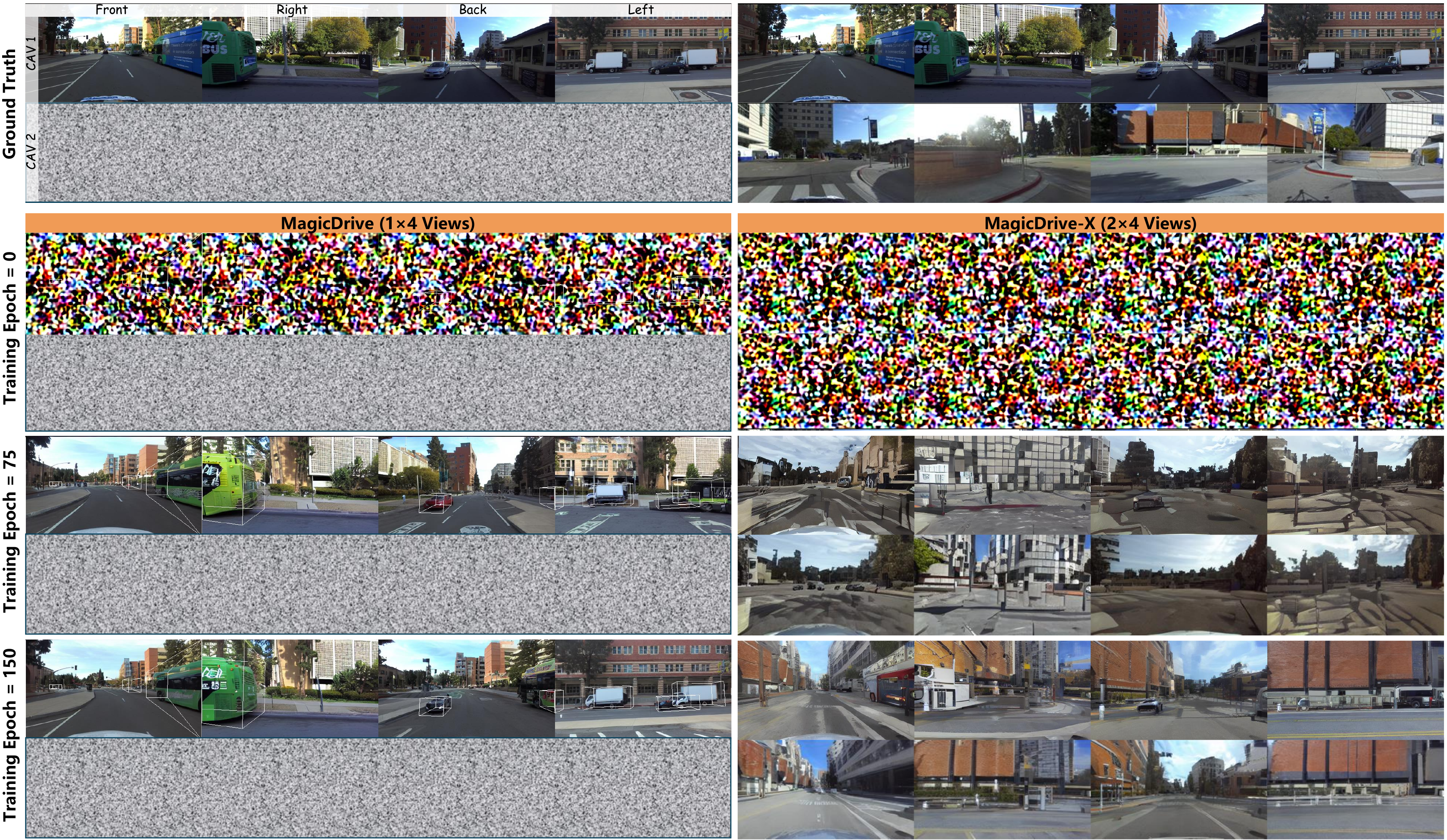}
    \vspace{-5mm}
    \caption{\textbf{Empirical evidence of training convergence failure in naive multi-agent extension.} While MagicDrive trained on single-agent data (left) converges well, the same model trained on dual-agent data (right) fails to converge and produce coherent results, demonstrating the necessity of a progressive training strategy.}
    \label{fig:scalability}
    \vspace{-6mm}
\end{figure*}

\subsection{Main Theoretical Statements}
\label{subsec:progressive_decomp}

Consider joint denoising over $N$ agents with $M$ cameras each. The coupled DDPM objective is
\begin{equation}
\label{eq:joint_loss}
\mathcal{L}_{\text{joint}} = \mathbb{E}_{z_{1:N,0}, \epsilon, t} \left[ \sum_{i=1}^N \left\|\epsilon_i - \epsilon_\theta(z_{1:N,t}, t, \mathcal{S}, \mathbf{P}_{1:N}, \mathcal{G})\right\|^2 \right],
\end{equation}
where $z_{1:N,t} = \{z_{1,t}, \ldots, z_{N,t}\}$ are the noisy latents for all $N \times M$ views at timestep $t$, $\epsilon_i$ is the target noise for agent $i$, $\mathcal{S}$ denotes the shared scene condition, $\mathbf{P}_{1:N}$ denotes the camera-pose conditions of all agents, and $\mathcal{G}$ denotes the geometric control signal. Let $\mathcal{R}_{\text{in}}$ and $\mathcal{R}_{\text{cr}}$ denote the intra-agent and cross-agent relation structures, let $\mathcal{L}_{\text{stage1}}$ and $\mathcal{L}_{\text{stage2}}$ denote the Stage-I and Stage-II objectives, let $\mathcal{C}(\cdot)$ denote the complexity surrogate, let $\mathcal{C}_{\text{couple}}$ denote the coupling term, and let $g_{\theta_1}, g_\phi$ denote the stochastic gradients of the two parameter blocks. The main statement below summarizes the two benefits of progressive training: lower surrogate complexity and decoupled stage-wise gradients.

\begin{theorem}[Progressive training improves local optimization]
\label{thm:progressive_advantage}
Under Assumptions~\ref{ass:noise_decomp}--\ref{ass:orthogonality}, progressive training reduces optimization difficulty in two senses. It removes the complexity gap caused by jointly fitting $\mathcal{R}_{\text{in}}$ and $\mathcal{R}_{\text{cr}}$ in one parameter space, and it eliminates cross-stage gradient interference. Specifically,
\begin{equation}
\label{eq:joint_complexity_repeat}
\Delta_{\mathcal{C}}
:=
\mathcal{C}(\mathcal{L}_{\text{joint}})
-
\left[\mathcal{C}(\mathcal{L}_{\text{stage1}}) + \mathcal{C}(\mathcal{L}_{\text{stage2}})\right]
=
\mathcal{C}_{\text{couple}}
> 0,
\end{equation}
\begin{equation}
\label{eq:zero_covariance}
\operatorname{Tr}\left(\operatorname{Cov}(g_{\theta_1}, g_\phi)\right) = 0,
\end{equation}
\end{theorem}

Theorem~\ref{thm:progressive_advantage} states that progressive training helps because it removes the coupling burden in the surrogate complexity and isolates the two parameter blocks across stages. A normalized reduction ratio is
\begin{equation}
\label{eq:complexity_reduction}
\rho_{\text{reduce}}
:=
\frac{\Delta_{\mathcal{C}}}{\mathcal{C}(\mathcal{L}_{\text{joint}})}
=
\frac{
2\sqrt{\mathcal{C}(\mathcal{R}_{\text{in}})\mathcal{C}(\mathcal{R}_{\text{cr}})}
}{
\mathcal{C}(\mathcal{R}_{\text{in}})
+
\mathcal{C}(\mathcal{R}_{\text{cr}})
+
2\sqrt{\mathcal{C}(\mathcal{R}_{\text{in}})\mathcal{C}(\mathcal{R}_{\text{cr}})}
},
\end{equation}
where $\rho_{\text{reduce}}$ measures the fraction of surrogate complexity removed by progressive decomposition.

\begin{corollary}[Implication for multi-agent generation]
\label{cor:interpretation}
If $\mathcal{C}(\mathcal{R}_{\text{in}})$ and $\mathcal{C}(\mathcal{R}_{\text{cr}})$ are comparable, then Eq.~\eqref{eq:complexity_reduction} yields a substantial complexity reduction, and Stage I is more stable than direct optimization of Eq.~\eqref{eq:joint_loss}. This matches Fig.~\ref{fig:scalability}.
\end{corollary}

\subsection{Definitions, Assumptions, and Supporting Lemmas}

This subsection introduces the formal objects used in the proof: the relation decomposition, the complexity surrogate, the stage-wise objectives, and the regularity conditions.

\begin{definition}[Intra-agent and cross-agent relations]
\label{def:relations}
Let $\mathcal{R}_{\text{in}}$ denote the collection of static intra-agent view relations induced by cameras on the same agent, and let $\mathcal{R}_{\text{cr}}$ denote the dynamic cross-agent relations induced by cameras across different agents.
\end{definition}

\begin{assumption}[Additive relation decomposition]
\label{ass:noise_decomp}
Assume the effective denoising target decomposes as
\begin{equation}
\label{eq:noise_decomp}
\epsilon_\theta(z, \mathcal{R}_{\text{in}}, \mathcal{R}_{\text{cr}})
=
\epsilon_\theta^{\mathcal{R}_{\text{in}}}(z, \mathcal{R}_{\text{in}})
+
\epsilon_\theta^{\mathcal{R}_{\text{cr}}}(z, \mathcal{R}_{\text{cr}})
+
\epsilon_{\text{res}},
\end{equation}
where $\epsilon_{\text{res}}$ collects residual interactions.
\end{assumption}

\begin{definition}[Complexity surrogate]
\label{def:complexity}
For a multi-view relation graph $\mathcal{R}$ with edge set $\mathcal{E}$, define the optimization-complexity surrogate
\begin{equation}
\label{eq:complexity_def}
\mathcal{C}(\mathcal{R}) = |\mathcal{E}| \cdot \mathbb{E}_{\text{frames}}[\operatorname{Var}(\mathcal{E})],
\end{equation}
where $|\mathcal{E}|$ measures the scale of relation modeling and $\operatorname{Var}(\mathcal{E})$ captures how much those relations vary across frames.
\end{definition}

\begin{assumption}[Positive coupling in joint learning]
\label{ass:positive_coupling}
Assume the joint complexity satisfies
\begin{equation}
\label{eq:joint_complexity}
\mathcal{C}(\mathcal{L}_{\text{joint}})
=
\mathcal{C}(\mathcal{R}_{\text{in}})
+
\mathcal{C}(\mathcal{R}_{\text{cr}})
+
\underbrace{2\sqrt{\mathcal{C}(\mathcal{R}_{\text{in}})\mathcal{C}(\mathcal{R}_{\text{cr}})}}_{\mathcal{C}_{\text{couple}} > 0},
\end{equation}
where $\mathcal{C}_{\text{couple}}$ is the coupling term.
\end{assumption}

The progressive strategy decomposes Eq.~\eqref{eq:joint_loss} into two sequential stages. Stage I isolates the static intra-agent term, whereas Stage II freezes the Stage-I backbone and learns only cross-agent refinement. Let $z_{i,0}$ denote the clean latent of agent $i$, let $\mathcal{G}_{\text{in}}(i)$ and $\mathcal{G}_{\text{cr}}(i)$ denote the intra-agent and cross-agent geometric conditions for agent $i$, and let $h_{-i}$ denote the frozen neighbor features. The corresponding objectives and refinement form are
\begin{equation}
\label{eq:stage1_loss}
\mathcal{L}_{\text{stage1}} =
\mathbb{E}_{z_{i,0}, \epsilon, t}
\left[
\left\|\epsilon - \epsilon_{\theta_1}(z_{i,t}, t, \mathcal{S}, \mathbf{P}_i, \mathcal{G}_{\text{in}}(i))\right\|^2
\right],
\end{equation}
\begin{equation}
\label{eq:stage2_loss}
\mathcal{L}_{\text{stage2}}
=
\mathbb{E}_{z_{i,0}, \epsilon, t}
\left[
\left\|\epsilon - \epsilon_{\theta_1,\phi}(z_{i,t}, \cdots, h_{-i}, \mathcal{G}_{\text{cr}}(i))\right\|^2
\right],
\end{equation}
\begin{equation}
\label{eq:residual_refinement}
\epsilon_{\theta_1,\phi}(z_{i,t}, \cdots, h_{-i})
=
\underbrace{\epsilon_{\theta_1}(z_{i,t}, \cdots)}_{\text{fixed intra-agent term}}
+
\underbrace{\Delta \epsilon_\phi(h_{-i}, \mathcal{G}_{\text{cr}}(i))}_{\text{cross-agent refinement}},
\end{equation}
where $\cdots$ abbreviates the conditioning variables shared with Eq.~\eqref{eq:stage1_loss}, $\operatorname{Enc}_{\theta_1}$ denotes the frozen Stage-I encoder, and $h_{-i} = \operatorname{Enc}_{\theta_1}(z_{-i,0})$ denotes the resulting neighbor features.

\begin{assumption}[Gradient conflict in joint optimization]
\label{ass:grad_conflict}
Let
\[
\bar{g}^{\text{in}} = \frac{1}{N}\sum_{i=1}^N \nabla_\theta \mathcal{L}_i^{\mathcal{R}_{\text{in}}},
\qquad
\bar{g}^{\text{cr}} = \frac{1}{N}\sum_{i=1}^N \nabla_\theta \mathcal{L}_i^{\mathcal{R}_{\text{cr}}}.
\]
where $\mathcal{L}_i^{\mathcal{R}_{\text{in}}}$ and $\mathcal{L}_i^{\mathcal{R}_{\text{cr}}}$ are the intra-agent and cross-agent loss components for agent $i$. Assume these gradients are non-aligned:
\begin{equation}
\label{eq:grad_conflict}
\cos(\bar{g}^{\text{in}}, \bar{g}^{\text{cr}})
=
\frac{\langle \bar{g}^{\text{in}}, \bar{g}^{\text{cr}} \rangle}
{\|\bar{g}^{\text{in}}\| \cdot \|\bar{g}^{\text{cr}}\|}
\leq 0.
\end{equation}
\end{assumption}

\begin{assumption}[Stage-wise parameter separation]
\label{ass:orthogonality}
During Stage II, $\theta_1$ is frozen and $\phi$ parameterizes only the refinement term. Then
\begin{equation}
\label{eq:grad_orthogonal}
\nabla_\phi \mathcal{L}_{\text{stage2}}
=
\mathbb{E}\left[\nabla_\phi \left\|\Delta \epsilon_\phi(h_{-i}, \mathcal{G}_{\text{cr}})\right\|^2\right],
\operatorname{Cov}(g_{\theta_1}, g_\phi) = 0.
\end{equation}
where $g_{\theta_1}$ and $g_\phi$ are the corresponding stochastic gradients.
\end{assumption}

\subsection{Proof of Supporting Lemmas}

This subsection proves the two technical ingredients used in Theorem~\ref{thm:progressive_advantage}: joint training suffers from gradient attenuation under conflict, whereas progressive decomposition removes the complexity coupling term.

\begin{lemma}[Joint training gradient attenuation]
\label{lem:grad_attenuation}
Under Assumption~\ref{ass:grad_conflict},
\begin{equation}
\label{eq:grad_magnitude}
\|\bar{g}^{\text{in}} + \bar{g}^{\text{cr}}\|^2
=
\|\bar{g}^{\text{in}}\|^2 + \|\bar{g}^{\text{cr}}\|^2 + 2\langle \bar{g}^{\text{in}}, \bar{g}^{\text{cr}} \rangle
\leq
\|\bar{g}^{\text{in}}\|^2 + \|\bar{g}^{\text{cr}}\|^2.
\end{equation}
\end{lemma}

\begin{proof}
The statement follows by directly expanding the squared norm. Since Assumption~\ref{ass:grad_conflict} gives $\langle \bar{g}^{\text{in}}, \bar{g}^{\text{cr}} \rangle \leq 0$, the cross term cannot increase the descent magnitude. Hence direct joint optimization suffers from gradient attenuation due to interference between static and dynamic signals. \end{proof}

\begin{lemma}[Complexity decomposition under progressive training]
\label{lem:complexity_reduction}
Under Assumptions~\ref{ass:noise_decomp} and~\ref{ass:orthogonality},
\begin{equation}
\label{eq:stage1_grad}
\nabla_{\theta_1} \mathcal{L}_{\text{stage1}}
=
\mathbb{E}\left[\nabla_{\theta_1}\left\|\epsilon - \epsilon_{\theta_1}(z_{i,t}, \cdots)\right\|^2\right],
\end{equation}
\begin{equation}
\label{eq:stage1_variance}
\operatorname{Var}(\nabla_{\theta_1}\mathcal{L}_{\text{stage1}})
\leq
\frac{\sigma^2}{B},
\qquad
\sigma^2 = \mathbb{E}\left[\left\|\epsilon - \epsilon_{\theta_1}\right\|^2\right],
\end{equation}
\begin{equation}
\label{eq:sequential_complexity}
\mathcal{C}(\mathcal{L}_{\text{stage1}}) + \mathcal{C}(\mathcal{L}_{\text{stage2}})
=
\mathcal{C}(\mathcal{R}_{\text{in}}) + \mathcal{C}(\mathcal{R}_{\text{cr}}).
\end{equation}
where $B$ is the batch size.
\end{lemma}

\begin{proof}
Eq.~\eqref{eq:stage1_grad} follows by differentiating the Stage-I objective. Since Stage I depends only on the frame-invariant intra-agent structure, its stochastic gradient variance satisfies Eq.~\eqref{eq:stage1_variance}. Moreover, Stage I optimizes the intra-agent component in Eq.~\eqref{eq:noise_decomp}, whereas Stage II optimizes only the residual cross-agent term on frozen intra-agent features. The two stages therefore avoid fitting both relation types in a shared parameter space, which removes the coupling term and yields Eq.~\eqref{eq:sequential_complexity}. \end{proof}

\subsection{Proof of Theorem~\ref{thm:progressive_advantage} and Corollary~\ref{cor:interpretation}}

\begin{proof}[Proof of Theorem~\ref{thm:progressive_advantage}]
The first claim follows from Lemma~\ref{lem:complexity_reduction} together with Assumption~\ref{ass:positive_coupling}. By Eq.~\eqref{eq:joint_complexity}, joint training contains the additional coupling term $\mathcal{C}_{\text{couple}}$, whereas Eq.~\eqref{eq:sequential_complexity} shows that this term is absent under progressive training. Substituting these identities yields Eq.~\eqref{eq:joint_complexity_repeat}; dividing both sides by $\mathcal{C}(\mathcal{L}_{\text{joint}})$ gives Eq.~\eqref{eq:complexity_reduction}.

For the second claim, Assumption~\ref{ass:orthogonality} implies that Stage II updates only $\phi$ while keeping $\theta_1$ fixed. Therefore, the covariance between the Stage-I and Stage-II gradients vanishes, which gives Eq.~\eqref{eq:zero_covariance}. In contrast, Lemma~\ref{lem:grad_attenuation} shows that joint training mixes the intra-agent and cross-agent gradients, thereby attenuating the effective descent direction. This establishes the theorem.
\end{proof}

\begin{proof}[Proof of Corollary~\ref{cor:interpretation}]
When $\mathcal{C}(\mathcal{R}_{\text{in}}) \approx \mathcal{C}(\mathcal{R}_{\text{cr}})$, the numerator in Eq.~\eqref{eq:complexity_reduction} remains a non-negligible fraction of the denominator. Consequently, removing the coupling term yields a substantive reduction in optimization complexity. Moreover, Eq.~\eqref{eq:stage1_variance} shows that the first stage has lower stochastic instability than the fully coupled objective. The convergence in Fig.~\ref{fig:scalability} is consistent with these implications.
\end{proof}

\section{More Implementation Details}
\label{app:impl_details}

\subsection{Experimental Setup Details}
\label{subsec:training_details}
Our model initializes its weights from the pre-trained Stable Diffusion v1.5 \citep{rombach2021highresolution} and only trains the newly added parameters. \texttt{V2VCrafter} was trained on a server equipped with 2 Intel(R) Xeon(R) Platinum 8200C CPUs and 8 NVIDIA A800-SXM4-80GB GPUs. We set the learning rate to 1e-4 and used a batch size of 4 per GPU. The model was trained for 150 epochs (100 for stage I and 50 for stage II), which took approximately 38 hours to complete. The output image resolution for each of the 8 camera views (from 2 CAVs) is 272x480. For inference, we use a CFG guidance scale of 1.5 and 50 diffusion steps. For the downstream collaborative 3D object detection task, we employ BEVFusion \citep{liu2022bevfusion} as the base detector for each vehicle and use AttnFuse \citep{xuOPV2VOpenBenchmark2022} for multi-agent feature fusion. For data augmentation experiments, we additionally evaluate V2VNet \citep{v2vnet} and Late Fusion \citep{11020620} baselines, testing both camera-only (C) and camera-LiDAR (C+L) configurations.

\subsection{Baseline Implementation Details}
\label{subsec:baseline_impl}
To the best of our knowledge, no prior work has addressed the task of CAD scene generation. Therefore, we establish baselines by adapting state-of-the-art single-agent driving scene generation models. We select \change{three} representative methods, MagicDrive \citep{gao2023magicdrive}, \change{MagicDrive-V2 \citep{gao2025magicdrive-v2}}, and BEVControl \citep{yang2023bevcontrolaccuratelycontrollingstreetview}, which are originally designed for single-vehicle six-view generation on Nuscenes dataset~\citep{Caesar_2020_CVPR}. We retain their original single-vehicle cross-view attention with fixed neighbor camera view pairs and extend them to the multi-agent generation setting. We denote these adapted versions as MagicDrive-X, \change{MagicDrive-V2-X}, and BEVControl-X, respectively, and retrain them on our processed V2X-Real dataset. Note that MagicDrive \change{and MagicDrive-V2} originally \change{do} not use FPV mask as an explicit control signal. However, we found that without FPV masks, \change{these models} completely \change{fail} to converge on V2X-Real. For fair comparison, we augment \change{them} with explicit FPV mask conditioning using the same masks generated in our preprocessing pipeline (see Appendix~\ref{subsec:fpv_mask_generation}).

\subsection{Data Selection Strategy}
\label{subsec:data_selection}
To ensure training data quality, we adopt a filtering strategy inspired by \citep{guo2025neptune-x}. We evaluate generated samples from two perspectives: semantic consistency by computing CLIP similarity between generated images and text descriptions, and layout accuracy by using a pre-trained ResNet classifier to verify object categories within bounding boxes. We then employ an active sampling strategy to prioritize challenging samples with higher training difficulty across multiple attributes including viewpoint, location, and environment. For multi-agent multi-view generation, we aggregate difficulty scores across all CAVs to select top-k instances that most effectively improve cross-agent consistency.

\subsection{Classifier-Free Guidance (CFG)}
\label{subsec:cfg}
We employ a selective CFG strategy where scene-level conditions such as camera poses and text descriptions are randomly dropped during training at rate $\gamma$, while object-level conditions including 3D boxes, BEV maps, and FPV masks are always retained. This ensures flexible control over scene appearance while maintaining strict spatial layout accuracy.

\subsection{Data Augmentation Details}
\label{subsec:data_aug}

\noindent \textbf{Camera-Only Augmentation.}
For camera-only (C) augmentation, we generate samples by recombining control conditions from different frames, including text prompts, BEV maps, camera poses, and 3D bounding boxes, while introducing moderate changes to scene layout and appearance. Following the explanation in V2X-Real~\citep{v2x-real}, V2V corridor scenes contain more diverse road layouts and larger slope changes, which make BEV-based camera solutions less stable in estimating 3D poses. We therefore train and test the Basic, Moderate, and Hard subsets separately so that harder samples do not destabilize optimization or obscure the effect of augmentation across difficulty levels.

\noindent \textbf{Camera-LiDAR Augmentation.}
For camera-LiDAR (C+L) augmentation, downstream detectors require strict geometric alignment between generated RGB images and the original LiDAR point clouds. We therefore keep the 3D box conditions fixed to the LiDAR observations and vary only appearance-related signals such as text prompts. In this setting, all subsets are mixed for training and testing, since the LiDAR branch provides an additional geometric anchor that stabilizes optimization even when harder pose configurations are included. This distinction makes the two augmentation protocols complementary rather than inconsistent.

\section{More Ablation Studies}
\label{app:ablation}

\subsection{Ablation on CFG Scale}
\label{subsec:cfg_ablation}

Table~\ref{tab:hyper_ablation} reports the FID and camera-only CoDetection score (mAP@30) obtained under different CFG scales. CFG$=1.5$ gives the best overall trade-off and matches the main-result setting in Table~\ref{tab:main_results}, achieving mAP@30$=27.88$ with competitive fidelity. Although CFG$=2.0$ yields a slightly lower FID, its controllability is weaker. Larger scales degrade both fidelity and controllability, indicating that overly strong guidance harms generation quality. We therefore use CFG$=1.5$ in the main experiments.

\subsection{Ablation on FPV-BEV Fusion Weight}
\label{subsec:fpv_bev_ablation}

As discussed in Section~\ref{sec:conditioning}, FPV masks and BEV maps provide complementary geometric cues: FPV primarily governs precise object placement, while BEV supplies global road topology. For this ablation, we fix the CFG scale to 1.5 and vary only the residual fusion weight $\eta$ in Eq.~(5). Table~\ref{tab:hyper_ablation} summarizes the results. A small fusion weight yields the best trade-off, with $\eta=0.1$ achieving the strongest overall result. When $\eta=0$, the model underuses global road context; when $\eta$ becomes too large, coarse BEV layouts begin to dominate the more precise local guidance from FPV masks, degrading both fidelity and CoDetection score.

\begin{table}[t]
    \centering
    \caption{\textbf{Ablation on CFG scale and FPV-BEV fusion weight.} \colorbox{red!20}{\rule{0pt}{0.7ex}\hspace{0.8em}} and \colorbox{yellow!20}{\rule{0pt}{0.7ex}\hspace{0.8em}} mark the best and second-best settings within each ablation block.}
    \vspace{-3mm}
    \label{tab:hyper_ablation}
    \small
    \renewcommand{\arraystretch}{1.0}
    \setlength{\tabcolsep}{4pt}
    \resizebox{\linewidth}{!}{
    \begin{tabular}{l|ccccc}
        \toprule
        \textbf{CFG Scale} & \textbf{1.2} & \textbf{1.5} (\textbf{Ours}) & \textbf{2.0} & \textbf{3.0} & \textbf{5.0} \\
        \midrule
        \textbf{FID $\downarrow$ / mAP@30 $\uparrow$} & 17.91 / 27.51 & \cellcolor{red!20}\textbf{17.46} / \textbf{27.88} & \cellcolor{yellow!20}17.31 / 26.43 & 18.01 / 24.94 & 20.54 / 20.62 \\
        \midrule
        \textbf{$\eta$ Value} & \textbf{0.0} & \textbf{0.05} & \textbf{0.1} (\textbf{Ours}) & \textbf{0.5} & \textbf{1.0} \\
        \midrule
        \textbf{FID $\downarrow$ / mAP@30 $\uparrow$} & 17.94 / 27.68 & \cellcolor{yellow!20}17.62 / 27.62 & \cellcolor{red!20}\textbf{17.46} / \textbf{27.88} & 18.52 / 25.96 & 19.35 / 21.01 \\
        \bottomrule
    \end{tabular}}
    \vspace{-4mm}
\end{table}

\section{Experimental Results}
\label{app:exp_results}

\subsection{Evaluation Metrics}
\label{subsec:eval_metrics}
This section describes the metrics adopted in our experiments, including mean Average Precision (mAP) for detection utility evaluation, Fr\'echet Inception Distance (FID) for image realism assessment, and three retrieval-based consistency metrics for jointly observed objects across agents.

\noindent\textbf{Mean Average Precision (mAP):} mAP is a standard metric for evaluating object detection performance. The computation follows these steps: First, predicted bounding boxes are classified as True Positives (TP) or False Positives (FP) by comparing their Intersection-over-Union (IoU) with ground truth boxes against a predefined threshold. Second, predictions are ranked by confidence scores to construct the Precision-Recall (P-R) curve. Third, the Average Precision (AP) for each object class is computed as the area under its P-R curve. Finally, mAP is obtained by averaging AP values across all classes. In our experiments with V2X-Real dataset, we report two variants: $\text{mAP}@{50}$ uses a single IoU threshold of 0.5, following the PASCAL VOC convention, while $\text{mAP}@{30}$ uses a threshold of 0.3, which is more suitable for collaborative perception scenarios where localization requirements may be relaxed due to communication constraints. These metrics reflect both the detection accuracy across object categories and the localization precision of the generated scenes.

\noindent\textbf{Fr\'echet Inception Distance (FID):} FID~\citep{heusel2017gans} is a widely adopted metric for evaluating generative models by measuring the statistical distribution discrepancy between generated and real images in feature space. The computation involves two stages: feature extraction and distribution distance calculation. First, FID extracts deep feature representations $f_{\text{gen}}$, $f_{\text{gt}}$ from generated and ground truth images using a pre-trained Inception-v3 network. Then, it computes the mean $\{\mu_{\text{gen}}, \mu_{\text{gt}}\}$ and covariance matrices $\{\Sigma_{\text{gen}}, \Sigma_{\text{gt}}\}$ of these features. The FID score is calculated as:
\begin{equation}
\text{FID} = \|\mu_{\text{gen}} - \mu_{\text{gt}}\|^2 + \text{Tr}\left(\Sigma_{\text{gen}} + \Sigma_{\text{gt}} - 2(\Sigma_{\text{gen}}\Sigma_{\text{gt}})^{1/2}\right),
\end{equation}
where $\|\cdot\|^2$ denotes the squared Euclidean norm and $\text{Tr}(\cdot)$ is the matrix trace operator. Lower FID indicates better generation quality with closer distribution alignment to real data.

\noindent\change{\textbf{CLIP Similarity:} To evaluate the semantic consistency of jointly observed objects across different agents, we first extract image crops of the shared objects using projected 2D bounding boxes, then encode these crops with the pre-trained CLIP ViT-B/32 model \citep{pmlr-v139-radford21a}. CLIP Similarity is computed as the cosine similarity between the averaged embeddings of the same object from two agents. Higher values indicate better cross-agent appearance consistency.}

\noindent\change{\textbf{Mean Reciprocal Rank (MRR):} CLIP Similarity only measures pairwise agreement for matched object pairs. To further evaluate whether the correct cross-agent pairs can be distinguished from all other shared objects in the same scene, we formulate a retrieval task. For each shared object crop from Agent A, we rank all shared object crops from Agent B by CLIP cosine similarity, and record the reciprocal rank of the true match. MRR averages this reciprocal rank over all shared objects. A higher MRR means the correct cross-agent match is ranked closer to the top more consistently.}

\noindent\change{\textbf{Top-1 Accuracy:} Based on the same retrieval protocol as MRR, Top-1 Accuracy measures the fraction of queries for which the correct cross-agent object is ranked first among all candidate shared objects from the other agent. Compared with CLIP Similarity and MRR, Top-1 is a stricter metric, since it requires exact rank-1 retrieval rather than just high similarity or near-top ranking. Higher Top-1 Accuracy indicates more reliable instance-level consistency across agents.}

\subsection{More Image Generation Results}
\label{subsec:more_gen}

\noindent\textbf{Random Seeds and Text-Free Generation.}
By changing random seeds or removing text prompts, our \texttt{V2VCrafter} can generate diverse CAD scenes while preserving cross-agent consistency on jointly observed objects. Fig.~\ref{fig:random_seed} shows that varying seeds changes foreground attributes and background appearance across CAV 1 and CAV 2, yet jointly observed objects marked by $\mathbf{V}^*$ still maintain identical appearance across agent perspectives. Fig.~\ref{fig:text_free} further demonstrates text-free conditional generation: even without text prompts, the model still synthesizes diverse object appearances and scene backgrounds under fixed geometric conditions.

\begin{figure*}[t]
    \vspace{-2mm}
    \centering
    \includegraphics[width=1.0\textwidth]{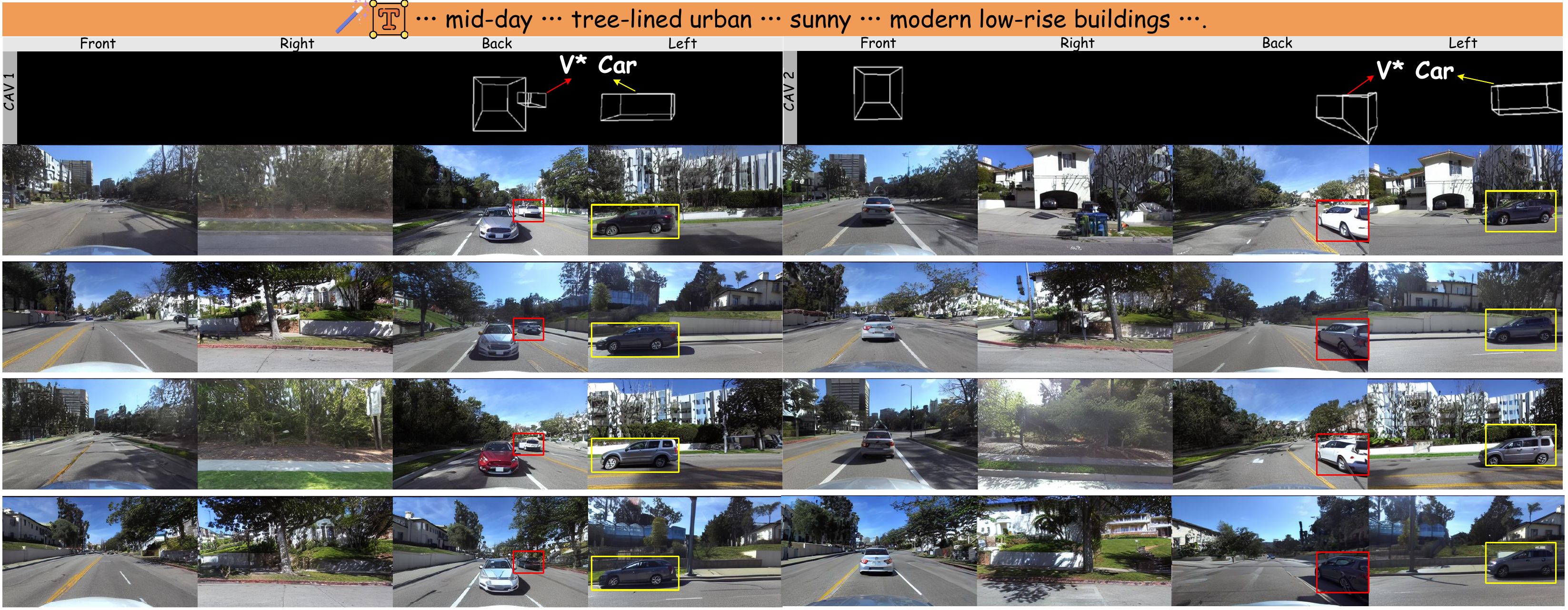}
    \vspace{-5mm}
    \caption{\textbf{Visualization of generated images with different random seeds}. With fixed text prompts and 3D bounding box conditions, varying random seeds produces diverse foreground object attributes and background appearance across CAV 1 and CAV 2. However, jointly observed objects marked by $\mathbf{V}^*$ consistently maintain identical visual appearance across agents.}
    \label{fig:random_seed}
    \vspace{-2mm}
\end{figure*}

\noindent\change{\textbf{More Baseline Comparison.} We supplement the main paper (Fig.~\ref{fig:gen_baselines}) with qualitative comparisons against MagicDrive-V2-X in Fig.~\ref{fig:magicdrive}. Note that MagicDrive-V2 is originally designed for video generation. We adapt it for single-frame synthesis by removing temporal layers while preserving its spatial control capabilities, ensuring the same generation resolution as our method. Although MagicDrive-V2-X generates plausible scenes, it fails to maintain cross-agent consistency for jointly observed objects, leading to conflicting semantic features. This visual discrepancy aligns with its lower CLIP similarity in Table~\ref{tab:main_results}, validating the necessity of our cross-agent attention for coherent multi-agent generation.}

\begin{figure}[t]
    \centering
    \includegraphics[width=1\columnwidth]{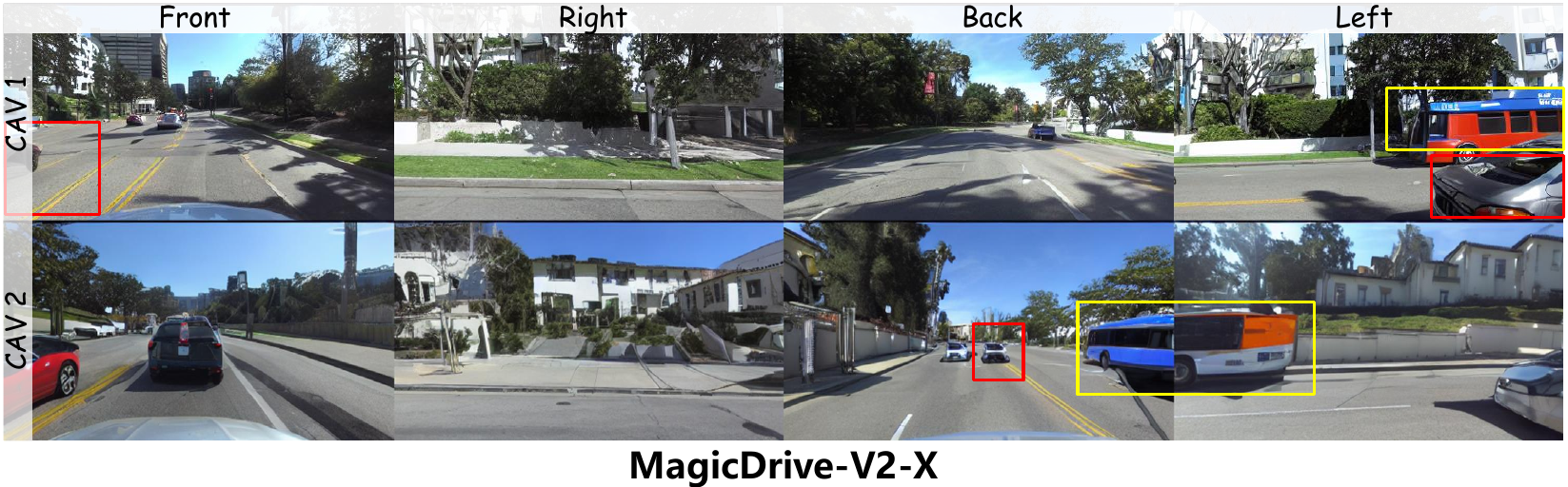}
    \vspace{-7mm}
    \caption{\textbf{Qualitative results of MagicDrive-V2-X}. Rectangles with the same color denote the same observed area.}
    \label{fig:magicdrive}
    \vspace{-6mm}
\end{figure}

\begin{figure*}[htbp]
    \vspace{-5mm}
    \centering
    \includegraphics[width=1.0\textwidth]{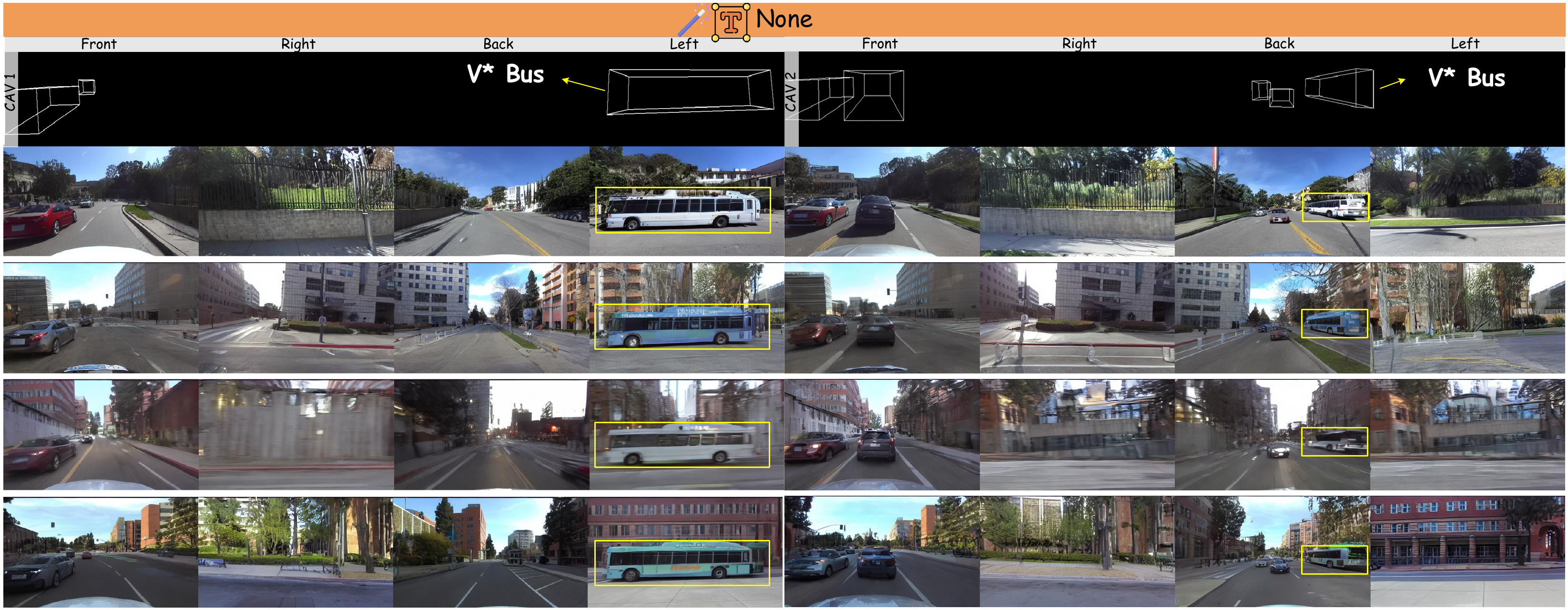}
    \vspace{-5mm}
    \caption{\textbf{Text-free conditional generation results}. By removing text prompts while keeping geometric conditions (3D boxes, BEV maps, camera poses), the model generates diverse appearances for both foreground objects (\textcolor{yellow}{yellow} boxes highlight the jointly observed $\mathbf{V}^*$ bus) and background scenes across two vehicles' multi-view observations, demonstrating the model's capability to synthesize varied physical attributes under identical layout constraints.}
    \label{fig:text_free}
    \vspace{-2mm}
\end{figure*}

\subsection{More Data Augmentation Results}
\label{subsec:more_collab}

\noindent \textbf{Camera+LiDAR Augmentation Results.}
Tables~\ref{tab:aug_cp} and \ref{tab:perclass_aug_cl} show that Camera+LiDAR augmentation consistently improves collaborative detection across different CP models, distance ranges, and object categories. The gains are particularly clear at longer ranges and for less frequent categories such as Van, Bus, and Truck, where sparse observations and limited data make learning more difficult. This suggests that the generated images provide complementary appearance diversity while preserving the geometric structure required by LiDAR-based fusion, making them especially valuable for harder cases without harming the overall detection behavior.

\begin{table}[t]
    \centering
    \scriptsize
    \caption{\textbf{Camera+LiDAR augmentation results.} \colorbox{red!20}{\rule{0pt}{0.7ex}\hspace{0.8em}} and \colorbox{yellow!20}{\rule{0pt}{0.7ex}\hspace{0.8em}} mark the largest and second-largest gains among \textit{+Gen Data} entries in the range-based columns.}
    \vspace{-3mm}
    \label{tab:aug_cp}
    \resizebox{\linewidth}{!}{
    \renewcommand{\arraystretch}{1.15}%
    \setlength{\tabcolsep}{3.2pt}
    \begin{tabular}{llcccc}
        \toprule
        \multirow{2}{*}[-0.6ex]{\textbf{CP Model}} & \multirow{2}{*}[-0.6ex]{\textbf{Data}} & \multicolumn{4}{c}{\textbf{Range-Based Results (mAP@30/50)} $\uparrow$} \\
        \cmidrule(lr){3-6}
         &  & \textbf{Overall} & \textbf{0-30m} & \textbf{30-50m} & \textbf{50-100m} \\
        \midrule
         & Real Only & 38.60 / 32.18 & \change{47.85 / 40.25} & \change{30.92 / 25.18} & \change{17.64 / 15.82} \\
        AttnFuse & \textit{+Gen Data} & \cellcolor{yellow!20}\textbf{41.71 / 34.27} & \change{\cellcolor{yellow!20}\textbf{50.14 / 42.08}} & \change{\cellcolor{yellow!20}\textbf{34.86 / 27.53}} & \change{\cellcolor{yellow!20}\textbf{21.75 / 18.29}} \\
         & Gain & \textcolor{green!60!black}{+3.11 / +2.09} & \change{\textcolor{green!60!black}{+2.29 / +1.83}} & \change{\textcolor{green!60!black}{+3.94 / +2.35}} & \change{\textcolor{green!60!black}{+4.11 / +2.47}} \\
        \midrule
         & Real Only & 44.15 / 38.45 & \change{54.82 / 46.53} & \change{35.41 / 29.54} & \change{20.23 / 19.52} \\
        V2VNet & \textit{+Gen Data} & \textbf{46.02 / 39.21} & \change{\textbf{56.19 / 47.08}} & \change{\textbf{37.72 / 30.56}} & \change{\textbf{22.65 / 20.69}} \\
         & Gain & \textcolor{green!60!black}{+1.87 / +0.76} & \change{\textcolor{green!60!black}{+1.37 / +0.55}} & \change{\textcolor{green!60!black}{+2.31 / +1.02}} & \change{\textcolor{green!60!black}{+2.42 / +1.17}} \\
        \midrule
         & Real Only & 29.84 / 20.75 & \change{37.12 / 25.53} & \change{23.94 / 16.52} & \change{13.53 / 10.51} \\
        Late Fusion & \textit{+Gen Data} & \cellcolor{red!20}\textbf{33.48 / 22.77} & \change{\cellcolor{red!20}\textbf{39.91 / 27.19}} & \change{\cellcolor{red!20}\textbf{28.37 / 18.84}} & \change{\cellcolor{red!20}\textbf{17.86 / 13.14}} \\
         & Gain & \textcolor{green!60!black}{+3.64 / +2.02} & \change{\textcolor{green!60!black}{+2.79 / +1.66}} & \change{\textcolor{green!60!black}{+4.43 / +2.32}} & \change{\textcolor{green!60!black}{+4.33 / +2.63}} \\
        \bottomrule
    \end{tabular}
    }
\end{table}

\begin{table}[t]
    \centering
    \scriptsize
    \caption{\textbf{Per-class Camera+LiDAR augmentation results.}}
    \vspace{-3mm}
    \label{tab:perclass_aug_cl}
    \resizebox{\linewidth}{!}{
    \renewcommand{\arraystretch}{1.1}%
    \setlength{\tabcolsep}{3.0pt}
    \begin{tabular}{lcccccc}
        \toprule
        \textbf{Data} & \textbf{Overall (mAP@30/50)} & \textbf{Car} & \textbf{Van} & \textbf{Bus} & \textbf{Truck} & \textbf{Pedestrian} \\
        \midrule
        Real Only & 38.60 / 32.18 & 59.18 / 55.82 & 19.35 / 16.48 & 24.52 / 21.17 & 26.74 / 23.18 & 32.64 / 18.35 \\
        +\texttt{V2VCrafter} & \textbf{41.71 / 34.27} & \textbf{61.24 / 57.74} & \textbf{23.18 / 19.74} & \textbf{28.15 / 24.52} & \textbf{30.12 / 26.48} & \textbf{34.18 / 19.82} \\
        \midrule
        Gain & \textcolor{green!60!black}{+3.11 / +2.09} & \textcolor{green!60!black}{+2.06 / +1.92} & \textcolor{green!60!black}{+3.83 / +3.26} & \textcolor{green!60!black}{+3.63 / +3.35} & \textcolor{green!60!black}{+3.38 / +3.30} & \textcolor{green!60!black}{+1.54 / +1.47} \\
        \bottomrule
    \end{tabular}
    }
    \vspace{-2mm}
\end{table}

\refstepcounter{table}\label{tab:jointly_observed_ratio}
\begin{wraptable}{r}{0.48\textwidth}
    \centering
    \vspace{-2mm}
    {\scriptsize
    \renewcommand{\arraystretch}{1.1}%
    \setlength{\tabcolsep}{3.2pt}
    \resizebox{\linewidth}{!}{
    \begin{tabular}{ll|c}
        \toprule
        \textbf{CP Model} & \textbf{Training Data} & \textbf{Jointly Observed AP@30/50}$\uparrow$ \\
        \midrule
        \multirow{4}{*}[-0.5ex]{AttnFuse} & Real Only (0\%) & 33.96 / 20.15 \\
         & +Gen Data (2\%) & 36.45 / 21.36 {\scriptsize \textcolor{green!60!black}{(+2.49 / +1.21)}} \\
         & +Gen Data (4\%) & \cellcolor{yellow!20}37.84 / 21.98 {\scriptsize \textcolor{green!60!black}{(+3.88 / +1.83)}} \\
         & +Gen Data (6\%) & \cellcolor{red!20}38.11 / 22.15 {\scriptsize \textcolor{green!60!black}{(+4.15 / +2.00)}} \\
        \bottomrule
    \end{tabular}
    }}
    \caption*{\textbf{Table~\thetable.} Impact of jointly observed object ratio on camera-only collaborative detection in the Basic subset. \colorbox{red!20}{\rule{0pt}{0.7ex}\hspace{0.8em}} and \colorbox{yellow!20}{\rule{0pt}{0.7ex}\hspace{0.8em}} mark the largest and second-largest gains among \textit{+Gen Data} entries.}
    \vspace{-2mm}
\end{wraptable}

\noindent\textbf{Jointly Observed Object Augmentation.}
The original training set contains only a small proportion of jointly observed objects. We therefore vary this ratio during camera-only augmentation and evaluate jointly observed detection on the Basic subset. Table~\ref{tab:jointly_observed_ratio} shows a clear positive correlation between the jointly observed ratio and jointly observed mAP, indicating that our generated samples preserve cross-agent consistency effectively. The key insight is that increasing the ratio of consistent cross-agent observations provides stronger supervision for learning feature alignment required by collaborative BEV detection.

\section{Extension to Video Generation}
\label{app:video_generation}

Inspired by recent works on extending street-view generation to driving videos~\citep{gao2023magicdrive, gao2025magicdrive-v2}, we extend \texttt{V2VCrafter} to multi-agent video generation by introducing spatio-temporal attention, retaining cross-agent attention at keyframes, and adding a lightweight video fine-tuning stage on short clips from V2X-Real. Multi-agent geometric conditions and the learnable modifier $\mathbf{V}^*$ are injected at keyframes to preserve consistent jointly observed objects across agents, while temporal attention interpolates the intermediate frames. As shown in Fig.~\ref{fig:video_generation}, this design already supports temporally coherent short multi-agent video generation, but long and consitent video generation is still limited by the relatively small video training set and therefore is not yet as stable as in the image setting.

\begin{center}
    \centering
    \includegraphics[width=1.0\textwidth]{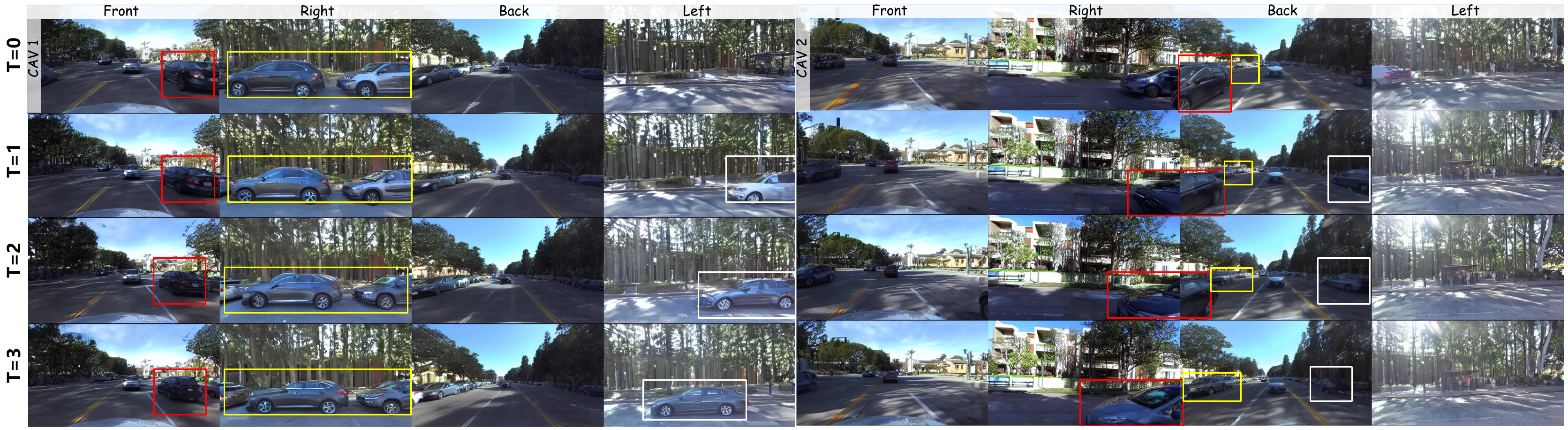}
    \vspace{-5mm}
    \captionof{figure}{\textbf{The video generation results of \texttt{V2VCrafter}}. We show the generation results of \texttt{V2VCrafter} on the V2X-Real dataset. The first and last frames are provided with 3D bounding box annotations, and the intermediate frames are generated by the model. The cross-agent jointly observed areas are highlighted with the same color.}
    \label{fig:video_generation}
    \vspace{-6mm}
\end{center}

\section{More Details of Data Preprocessing}
\label{app:data_preprocessing}
To prepare the data for training \texttt{V2VCrafter}, we perform five main preprocessing steps on V2X-Real: (1) scene filtering to select high-quality multi-vehicle scenarios, (2) pose-difficulty stratification for camera-only augmentation analysis, (3) FPV mask generation for fine-grained object-level control, (4) BEV map generation from LiDAR point clouds using OpenCOOD~\citep{xuOPV2VOpenBenchmark2022}, and (5) text description generation via Qwen2.5-VL-7B-Instruct~\citep{bai2025qwen25vltechnicalreport}. This section provides detailed technical descriptions of each preprocessing component.

\subsection{Scene Filtering}
\label{subsec:scene_filtering}
We focus exclusively on V2V data and exclude views from stationary RSUs, as RSU perspectives exhibit limited diversity and their image style does not align with the CAV viewpoint. To ensure meaningful cross-agent generation, we retain only scenes containing at least two valid CAV streams and further filter frames to require complete sensor data, including cameras, LiDAR, and well-formed YAML annotations, for both CAVs. We additionally apply temporal consistency checks with a fixed load interval to discard isolated frames and preserve continuous multi-agent driving sequences.

\subsection{Pose-Difficulty Stratification}
\label{subsec:pose_difficulty}
For the camera-only experiments in the main paper, we further stratify the filtered dataset into Basic, Moderate, and Hard subsets to reflect different levels of pose difficulty. The split is obtained from a sample-level score computed from information already provided by V2X-Real, including annotated 3D boxes, camera calibration, ego poses, and LiDAR. Specifically, we project each 3D box into all camera views to estimate the observation distance of jointly observed objects and their visibility degradation from truncation and overlap, and we use LiDAR together with the ego pose to estimate local road-slope variation. These normalized cues are averaged into a single score, where lower, intermediate, and higher scores correspond to the Basic, Moderate, and Hard subsets, respectively.

\subsection{FPV Mask Generation}
\label{subsec:fpv_mask_generation}
We generate camera-specific FPV semantic masks by projecting each annotated 3D bounding box into every camera view and rasterizing the visible polygon region with its semantic class ID. After obtaining the 2D projected corners, we use OpenCV's \texttt{convexHull} and \texttt{fillPoly} functions to produce pixel-accurate masks for the visible object regions. For each CAV, the masks from its four camera views are stacked into a 4-channel array, which allows the model to process multi-view object-level control signals efficiently during training.

\subsection{BEV Map Generation}
\label{subsec:bev_generation}
The original V2X-Real dataset does not provide BEV semantic maps required for spatial layout control. We generate these maps by rendering LiDAR point clouds and 3D bounding box annotations into bird's-eye view representations using tools from OpenCOOD~\citep{xuOPV2VOpenBenchmark2022}.

\noindent\textbf{Coordinate Transformation.} The generation process begins with transforming 3D point clouds from the vehicle's local coordinate frame to a standardized BEV coordinate system. Given a LiDAR point cloud $\mathcal{P} = \{p_i \in \mathbb{R}^3\}$ and the vehicle's 6-DOF pose $(x, y, z, \text{roll}, \text{yaw}, \text{pitch})$, we first construct the LiDAR-to-world transformation matrix $\mathbf{T}_{\text{L2W}} \in SE(3)$ by composing rotation matrices for yaw, pitch, and roll with the translation vector. The world-to-BEV transformation then projects points onto the ground plane while centering the coordinate system at the ego vehicle's position.

\noindent\textbf{Spatial Discretization.} We adopt a 200$\times$200 pixel BEV grid covering a spatial range of [-50m, 50m] in both X and Y directions, yielding a resolution of 0.5m per pixel. For each LiDAR point $p_i = (x_i, y_i, z_i)$, we compute its grid coordinates as:
\begin{equation}
u_i = \lfloor (x_i - x_{\min}) / \Delta \rfloor, \quad v_i = \lfloor (y_i - y_{\min}) / \Delta \rfloor,
\end{equation}
where $x_{\min} = y_{\min} = -50$m and $\Delta = 0.5$m is the spatial resolution.

\noindent\textbf{BEV Rendering.} We render the BEV map using the OpenCOOD visualization utilities with offline rendering to ensure consistent quality. Specifically, we use the \texttt{visualize\_single\_sample\_dataloader} function to process LiDAR points and 3D bounding boxes into Open3D geometries. The rendering pipeline includes: (1) colorizing point clouds based on intensity values or height information via the \texttt{color\_mode} parameter, (2) projecting 3D bounding boxes onto the ground plane as wireframe line sets, and (3) rendering the scene from a top-down view using either Open3D's \texttt{OffscreenRenderer} (with camera positioned at [0, 0, 50] looking down at [0, 0, 0]) or matplotlib's scatter plot backend for headless server environments. The final output is a 200$\times$200 RGB image where point clouds appear as colored dots and vehicle bounding boxes as red wireframes, providing clear spatial layout information for the diffusion model.

\subsection{Driving Scene Description Generation}
\label{subsec:qwen}
We use \textbf{Qwen2.5-VL-7B-Instruct} to generate hierarchical driving-scene descriptions for V2X-Real, with the prompt design illustrated in Fig.~\ref{fig:gen_scene_descriptions}. For each sample, we first generate a scene-level description from images sampled across the scene to summarize relatively stable attributes such as weather, time of day, and road environment, and then generate a frame-level description from the current multi-view images to capture dynamic objects and local traffic interactions. The final text prompt follows the template \verb|"{scene_description} {frame_description}"|, which combines global context with frame-specific details for controllable generation.

\section{More Discussion}
\label{app:more_discussion}

\subsection{Potential Applications}

\noindent\change{Beyond data augmentation, \texttt{V2VCrafter} can support several broader applications. Its ability to synthesize cross-agent consistent observations makes it useful for realistic V2X driving simulation on top of lightweight simulators such as CARLA, for providing multi-view supervision to V2X-oriented 3D reconstruction such as feedforward 3DGS, and for building predictive world models for collaborative planning~\citep{10107755, 10437455, 10614333, Min_2024_CVPR, Zheng_2025_ICCV}. These directions all benefit from the same core capability: controllable and spatially consistent generation across multiple agents.}

\begin{figure*}[htbp]
    \raggedright
    \begin{tcolorbox}[
        enhanced,
        width=\textwidth,
        colback=white,
        colframe=blue!40!gray,
        colbacktitle=blue!35!gray,
        coltitle=white,
        title=Frame-Level Description Prompt,
        fonttitle=\bfseries\large,
        halign title=flush left,
        boxrule=1pt,
        arc=3mm,
        left=4mm,
        right=4mm,
        top=2mm,
        bottom=2mm,
        before upper=\raggedright,
        drop shadow=black!12!white
    ]
    \begin{lstlisting}[style=mystyle, escapeinside={@}{@}, breakindent=0pt, breakautoindent=false]
Describe the visible objects and their immediate traffic situation in this frame from vehicle {vehicle_id}. Focus on the actions and positions of cars, trucks, and pedestrians. Generate a description of exactly 20-25 English words.

@\textbf{Objects detected}@: {context_info}

Write one sentence describing the @\textbf{key objects}@ and their @\textbf{current actions}@.

@\textbf{Example}@: "A white SUV is turning left at an intersection while a red sedan waits at the traffic light and several pedestrians are crossing the street."

@\textbf{Your description}@:
\end{lstlisting}
    \end{tcolorbox}
    \vspace{2mm}
    \begin{tcolorbox}[
        enhanced,
        width=\textwidth,
        colback=white,
        colframe=blue!40!gray,
        colbacktitle=blue!35!gray,
        coltitle=white,
        title=Scene-Level Description Prompt,
        fonttitle=\bfseries\large,
        halign title=flush left,
        boxrule=1pt,
        arc=3mm,
        left=4mm,
        right=4mm,
        top=2mm,
        bottom=2mm,
        before upper=\raggedright,
        drop shadow=black!12!white
    ]
    \begin{lstlisting}[style=mystyle, escapeinside={@}{@}, breakindent=0pt, breakautoindent=false]
Describe the overall style of this traffic scene in exactly 25-30 English words, based on the provided images from different times.

Summarize the general @\textbf{road type}@, @\textbf{weather conditions}@, @\textbf{time of day}@, and @\textbf{surrounding environment}@.

@\textbf{Example}@: "A daytime scene on a multi-lane urban road under clear, sunny skies, surrounded by modern high-rise buildings."

@\textbf{Your description}@:
\end{lstlisting}
    \end{tcolorbox}
    \caption{The structured prompts used for generating frame-level (top) and scene-level (bottom) descriptions. These prompts guide the Qwen2.5-VL-7B-Instruct model to produce consistent and relevant textual data.}
    \label{fig:gen_scene_descriptions}
\end{figure*}

\end{document}